\documentclass[11pt]{article}
\usepackage[margin=1in]{geometry}

\usepackage[T1]{fontenc}
\usepackage[utf8]{inputenc}
\usepackage[round]{natbib}
\usepackage{graphicx}
\usepackage[font=footnotesize]{subfig}
\usepackage{amsmath,amssymb,amsfonts,amsthm,mathtools}
\usepackage{booktabs}
\usepackage{multirow}
\usepackage{url}
\usepackage{enumitem}
\usepackage{algorithm}
\usepackage{algpseudocode}
\usepackage{xcolor}
\usepackage[hidelinks]{hyperref}
\usepackage{microtype}

\usepackage{wrapfig}

\newtheorem{theorem}{Theorem}[section]
\newtheorem{lemma}[theorem]{Lemma}
\newtheorem{proposition}[theorem]{Proposition}
\newtheorem{corollary}[theorem]{Corollary}

\theoremstyle{definition}
\newtheorem{definition}[theorem]{Definition}
\newtheorem{assumption}[theorem]{Assumption}

\theoremstyle{remark}

\DeclareMathOperator{\Law}{Law}
\DeclareMathOperator{\softmax}{softmax}

\newcommand{\ind}[1]{\mathbf{1}\!\left\{#1\right\}}
\newcommand{\ball}[2]{\mathcal{B}_{#1}(#2)}

\newcommand{\R}{\mathbb{R}}

\title{Ball Differential Privacy - \\ How to Mitigate Data Reconstruction with Less Noise}

\author{Joseph Margaryan and Nirupam Gupta\\[0.5ex]
\small Department of Computer Science, University of Copenhagen\\
\small Copenhagen, Denmark\\
\small \texttt{josephmargaryan@gmail.com}, \texttt{nigu@di.ku.dk}}
\date{}

\newcommand{\suchthat}{\ensuremath{~\middle|~}}
\newcommand{\knowing}{\suchthat{}}

\newcommand{\norm}[1]{\left\lVert{#1}\right\rVert}

\newcommand{\indexvar}[3]{\ensuremath{{{#3}^{\ifthenelse{\equal{#1}{}}{}{\left({#1}\right)}}_{#2}}}}
\newcommand{\indexvarNoPar}[3]{\ensuremath{{{#3}^{\ifthenelse{\equal{#1}{}}{}{\left{#1}\right}}_{#2}}}}

\providecommand{\norm}[1]{\ensuremath{\left\lVert#1\right\rVert }}

\newcommand{\proba}[2]{\ensuremath{\text{P}\!\left({#1}\ifthenelse{\equal{#2}{}}{}{\knowing{}{#2}}\right)}}

\renewcommand{\paragraph}[1]{\textbf{#1}~}

\begin{document}
\maketitle

\begin{abstract}
Vector embeddings of raw records, while not human-readable, do not preserve privacy of records: an adversary can reconstruct training records from a released model even when that model is a simple convex classifier. Differential privacy (DP) is the principled defense, but its noise is calibrated to the worst-case indistinguishability; hiding arbitrary single-record substitutions, including those far outside the set of plausible alternatives relevant to a reconstruction adversary. The result is noise far larger than what reconstruction robustness requires, degrading accuracy without a corresponding security benefit.

We propose \emph{Ball-DP}: enforcing $(\varepsilon,\delta)$ indistinguishability over single-record substitutions restricted to a ball of radius $r$ as per a distance metric $d$ in the embedding space. 
A deployment facing only local reconstruction threat can choose a small $r$, 
thereby reduce noise and recover accuracy. The radius makes the scope of the privacy claim explicit against reconstruction attacks; standard DP is recovered when $r$ covers the entire admissible record domain. 
We provide noise calibrations for regularized convex learning problems for Ball-DP, and derive the corresponding
reconstruction-robustness certificates (named {\em Ball-ReRo}) -- 
upper-bound on an attacker's reconstruction success.
By deriving the optimal finite-prior MAP reconstruction attack, we present empirical auditing of Ball-ReRo certificates for seven benchmark learning tasks. 
Our experiments show that calibrating noise to Ball-DP yields improvement in utility, considerably exceeding the dilution of reconstruction robustness in high privacy regimes, i.e., when $\varepsilon$ is small.

\end{abstract}

\section{Introduction}
Modern machine learning systems rarely operate directly on raw data. Instead, representation-learning pipelines commonly map text, images, or other types of user records to vector embeddings that can be reused by downstream task-specific classifying heads~\citep{bengio2013representation,devlin2019bert,reimers2019sentencebert,radford2021clip}. This representation layer creates a false sense of privacy. Embeddings may not be human-readable, yet a model trained on them can reveal information about the training records.
In the informed-adversary model, for example, an attacker knowing all but one training record can identify or reconstruct training records from a released model, even if the release is a simple convex model rather than a memorizing neural network \cite{balle2022informed}.

Privacy-preserving AI requires guarantees that remain meaningful even against strong adversaries with side information. Differential privacy (DP) provides such a guarantee by requiring the released model to look nearly the same under arbitrary changes to a single training record~\citep{cummings2024advancing}. The parameters $(\varepsilon,\delta)$ control the strength of this indistinguishability requirement, with smaller values giving stronger protection. A common implementation is output perturbation: train the model, then add enough noise to mask the effect of any single record~\citep{chaudhuri2011erm}. This worst-case calibration is simple and broadly applicable, but it can require substantial noise and therefore reduce accuracy when strong privacy is desired.

DP is naturally aligned with membership inference attacks (MIA), which ask whether a particular record was used to train the model: by making the released model nearly indistinguishable under any single-record changes, DP limits what an attacker can infer about involvement of specific record. However, participation is not always the sensitive fact of interest~\citep{hayes2023bounding,cummings2024attaxonomy}. In many settings, it may already be known that a person contributed data; the real privacy concern is whether the released model reveals the private content of their record.



This motivates training-record reconstruction attacks, where an adversary uses side information to recover the hidden record itself, often by choosing from a set of plausible candidates. The connection between DP guarantee and reconstruction risk is, however, quite loose~\citep{hayes2023bounding, cummings2024attaxonomy}. DP mechanisms provide indistinguishability against worst-case single-record substitutions that can be arbitrarily far from each other in the embedding space. While a reconstruction adversary is interested in generating an estimate (i.e., choosing an element from the candidate set) that is in proximity to the target record. Thus, to prevent reconstruction, it suffices to induce indistinguishability over single-record substitutions that are close to each other in the underlying embedding space.

\begin{wrapfigure}{l}{0.5\textwidth}
  \begin{center}
    \includegraphics[width=0.48\textwidth]{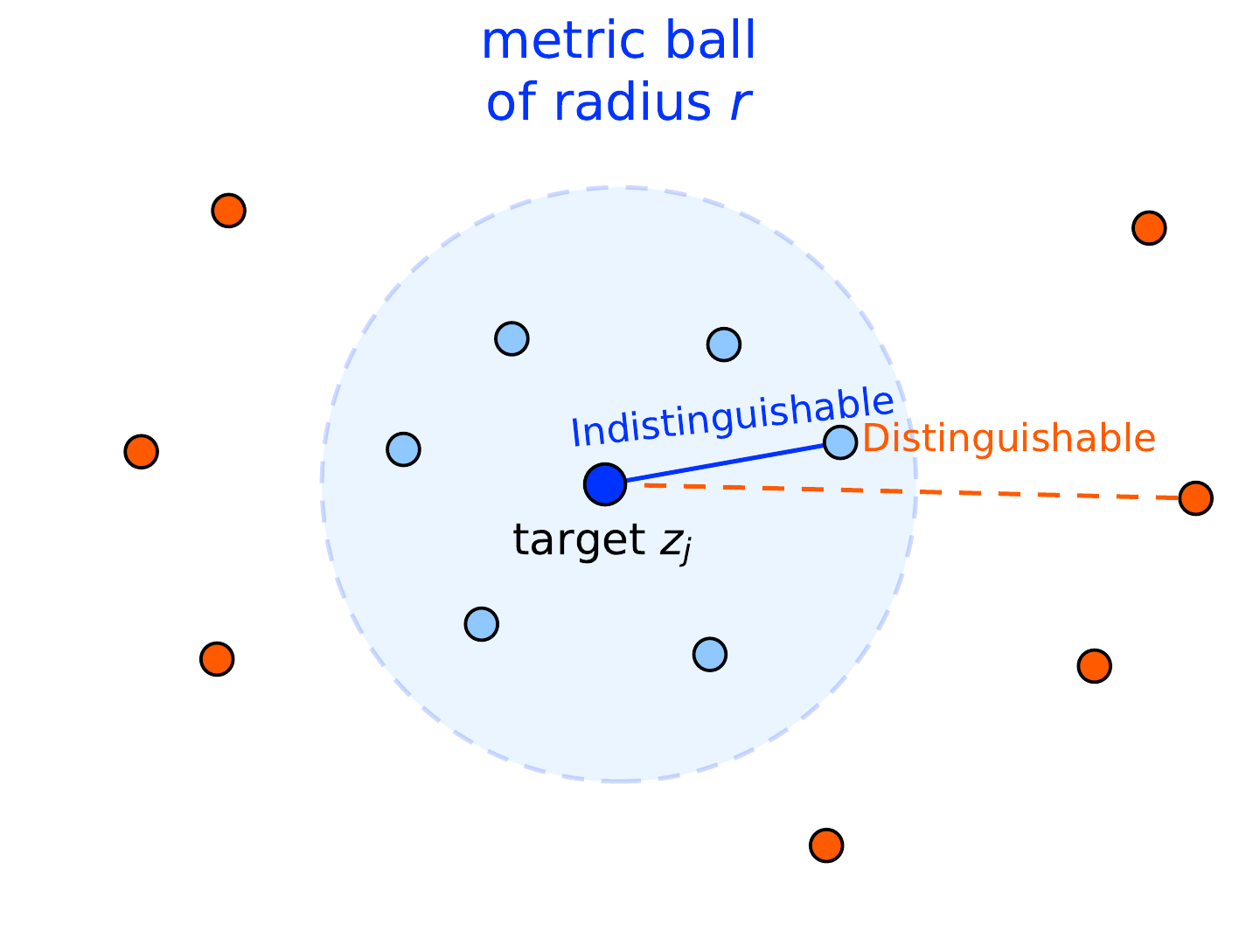}
  \end{center}
  \caption{\small Illustration of reconstruction protection due to Ball-DP: an adversary can narrow down the target record to points highlighted in blue but cannot reliably tell any further. By limiting the scope of $(\varepsilon, \delta)$-indistinguishability to replacements within a metric ball of radius $r$, Ball-DP mitigates accurate reconstruction while recovering higher model accuracy than standard DP.}
  \label{fig:ball-DP}
\end{wrapfigure}

\noindent \paragraph{Ball-DP.} We propose \emph{Ball-DP}, which deploys the usual $(\varepsilon,\delta)$-indistinguishability but applies it to only single-record substitutions within the ball of declared radius $r$. The radius becomes an explicit privacy parameter: smaller $r$ efficiently mitigates local reconstruction threats using less noise, while larger $r$ expands the protected set at a corresponding utility cost. For the regularized convex prediction heads that we consider, we can ensure $(r, \varepsilon, \delta)$-Ball-DP via Gaussian output perturbation whose noise scales with $r$ in addition to the indistinguishability parameters $(\varepsilon, \delta)$. 
As shown in Figure~\ref{fig:ball-DP}, Ball-DP allows an adversary to differentiate between the target record and a candidate outside the ball of radius $r$, by enforcing $(\varepsilon,\delta)$-indistinguishability over the ball it forces the adversary to incur an error that scales with $r$ in estimating the target record. We can tune $r$ according to our privacy requirement; recovering the standard DP guarantee when $r$ is large enough to cover the entire admissible record domain. \\

\noindent \paragraph{Contributions.} We make the following contributions.

\begin{enumerate}
    \item We define $(r, \varepsilon, \delta)$-Ball-DP formally, i.e., $(\varepsilon, \delta)$-indistinguishability over $r$-adjacency, 
    and establish its connection to standard DP. 
    \item We derive $\ell_2$-sensitivity of strongly convex empirical-risk minimization (ERM) under $r$-adjacency, and present the Gaussian release mechanism calibrated to Ball-DP guarantee. 
    \item We formally connect Ball-DP to reconstruction robustness; bounding an attacker's reconstruction success rate as a function of the oblivious prior and the privacy parameters $(r,\varepsilon,\delta)$.
    \item We empirically validate our theoretical bounds, and compare the accuracy-reconstruction robustness trade-off of Ball-DP and standard DP across seven learning benchmarks. 
    
\end{enumerate}

Our empirical evidence suggests that Ball-DP provides comparable reconstruction robustness while requiring less noise, thereby mitigating the accuracy--reconstruction robustness trade-off. At $\varepsilon=2$ and $\delta=10^{-6}$, Ball-DP improves mean accuracy over the standard global-radius DP comparator across all seven embedding benchmarks at the reported local policy radii. For example, at the $q50$ label-preserving policy radii, Ball-DP improves mean accuracy by $0.57$--$36.47$ percentage points. Even at the broader $q95$ local radii, Ball-DP remains more accurate than the global-radius comparator on every benchmark. Finite-prior maximum a posteriori (MAP) audits show comparable empirical exact-identification rates for candidate sets of size $m=10$.\\

\paragraph{Comparisons with metric-DP} We emphasize that Ball-DP should not be confused with metric-DP, or $d_X$-privacy~\citep{chatzikokolakis2013metrics,imola2022metricdp}. In its usual form, metric-DP imposes a distance-dependent privacy-loss bound,
so privacy degrades continuously with the distance between records. By contrast, $(r,\varepsilon,\delta)$-Ball-DP uses the metric only to define a binary adjacency relation $D\sim_r D'$, and then enforces the standard $(\varepsilon,\delta)$ indistinguishability inequality uniformly over that restricted adjacent subset of records. Ball-DP need not provide any guarantee for substitutions outside the radius, nor does it require privacy loss to shrink for pairs that are closer than $r$. The advantage of this formulation is that it separates the privacy policy from the mechanism calibration. In particular, mechanisms whose privacy analyses depend on the adjacency relation only through a sensitivity bound can be adapted by replacing the usual sensitivity to $r$-adjacency sensitivity. We show this explicitly in the following wherein we adapt the standard Gaussian mechanism to Ball-DP.\\

\paragraph{Paper organization.}
The rest of the paper is organized as follows.
Section~\ref{sec:preliminaries} fixes the ERM setup, the informed reconstruction threat model, and the Gaussian output-perturbation notation used throughout.
Section~\ref{sec:ball-dp} introduces Ball-DP as a radius-restricted adjacency policy and gives the corresponding Gaussian calibration for strongly convex ERM.
Section~\ref{sec:gaussian-rero} translates these privacy guarantees into Ball-ReRo reconstruction certificates and states the finite-prior MAP audit used in the experiments.
Section~\ref{sec:experiments} reports the utility and reconstruction-audit results on fixed embedding benchmarks.
Section~\ref{sec:technical-statements} collects the model-specific technical statements needed for the experimental calibration.
Section~\ref{sec:related} discusses related work on metric privacy, private convex ERM, reconstruction robustness, and prototype-based learning.
Finally, Section~\ref{sec:conclusion} provide a summary of the paper. Proofs and auxiliary derivations are deferred to the appendices.

\section{Preliminaries: ERM, Reconstruction, and Gaussian Releases}
\label{sec:preliminaries}

This section formalizes the learning problem, the reconstruction threat model, and the Gaussian-release notation used throughout the paper. The order is as follows: we first specify what is learned from embeddings, then describe why reconstruction is a meaningful privacy breach, and lastly discuss DP and Gaussian output-perturbation guarantees.\\

\paragraph{Embedding ERM setup.}
Let raw examples live in an application-dependent space $\mathcal X_0$, e.g., photos, text and tabular records, and let labels lie in space $\mathcal Y$. A public encoder $\phi:\mathcal X_0\to\mathbb R^d$
maps each raw example to an embedding. From this point onwards, a training record is
\[
z_i=(x_i,y_i)=(\phi(a_i),y_i)\in\mathcal Z:=\mathbb R^d\times\mathcal Y,
\]
where $a_i\in\mathcal X_0$ is the corresponding raw example. The label space may either be binary, $\mathcal Y=\{0,1\}$, or multiclass, $\mathcal Y=\{1,\ldots,K\}$. A dataset is an ordered tuple $D=(z_1,\ldots,z_n)\in\mathcal Z^n$. We consider $\ell_2$-regularized empirical risk minimization (ERM) on the training records. For a parameter $\theta\in\mathbb R^p$, a loss function $\ell(\theta;z)$, and regularization parameter $\lambda>0$, we define
\begin{equation}
\label{eq:generic-erm}
F_D(\theta)
:=
\frac{1}{n}\sum_{i=1}^n \ell(\theta;z_i)
+
\frac{\lambda}{2}\|\theta\|_2^2,
\quad 
\widehat\theta(D)
:=
\arg\min_{\theta\in\mathbb R^p} F_D(\theta).
\end{equation}
When $\ell(\cdot;z)$ is convex for every record $z$, the objective in~\eqref{eq:generic-erm} is $\lambda$-strongly convex, hence the minimizer is unique. This is a standard convex-learning template behind regularized linear and generalized-linear heads~\citep{hastie2009elements,shalev2014understanding}.\\

\paragraph{Reconstruction threat model.}
Embedding-space training does not by itself eliminate record-level privacy risk. We consider an {\em informed reconstruction adversary}, following the setting proposed and studied in~\cite{balle2022informed,hayes2023bounding}, that observes the released model and has substantial side information about the training set. In particular, all records except one target record $z$ are treated as known; when needed below, we denote this ordered tuple of known records by $D^-\in\mathcal Z^{n-1}$. The adversary also knows the training algorithm, hyperparameters, release rule, and any public information used to form beliefs about the hidden record.

The relevant reconstruction task is therefore not an unconstrained search over the full record space. Side information typically narrows the target to a finite shortlist, or more generally a prior distribution, over plausible alternatives~\citep{balle2022informed}. A privacy breach occurs when the released model lets the adversary identify the hidden record, or a sufficiently close substitute, better than what could be achieved otherwise from its prior information alone. For example, if the adversary has narrowed the target to $m$ plausible candidates and has no further information, an oblivious exact-identification guess succeeds with probability $1/m$; a meaningful reconstruction audit should be interpreted relative to this baseline. \\



\paragraph{Differential privacy.}
The gold standard for privacy-preservation of training records is differential privacy~\citep{dwork2014book}. 
Two datasets $D,D'\in\mathcal Z^n$ are adjacent, written $D\sim D'$, if they comprise identical records except one. The two differing records must lie in the admissible record domain. A randomized mechanism $M:\mathcal Z^n\to\Theta$ is said to satisfy $(\varepsilon,\delta)$-differential privacy if for every adjacent pair $D\sim D'$ and every measurable set $S\subseteq\Theta$,
\begin{equation}
\label{eq:dp-definition}
\Pr[M(D)\in S]
\le
 e^\varepsilon \Pr[M(D')\in S]+\delta.
\end{equation}
This indistinguishability prevents any adversary from inferring the inclusion of any record in the training. By extension, DP prevents even an informed adversary, described above, from extracting the hidden record with certainty. \\

\paragraph{Gaussian output perturbation.}
In convex learning, we can ensure DP by output perturbation. Specifically, instead of releasing the ERM output $\widehat\theta(D)$, we release a perturbed model: 
\[
\widetilde\theta(D)=\widehat\theta(D)+\xi,
\qquad
\xi\sim\mathcal N(0,\sigma^2 I_p).
\]
Here, $\mathcal  N(0,\sigma^2 I_p)$ denotes a Gaussian distribution with mean $0$ and covariance matrix $\sigma^2 I_p$. This is commonly referred to as the Gaussian mechanism applied to the trained parameter vector. The variance $\sigma^2$ scales with the $\ell_2$ sensitivity of the original unperturbed model~\citep{dwork2014book,chaudhuri2011erm,balle2018analytic}:
\[
\Delta_2:=\sup_{D\sim D'}\|\widehat\theta(D)-\widehat\theta(D')\|_2.
\]
Specifically, when
\(
\sigma \geq \frac{\Delta_2}{\varepsilon}\sqrt{2\log(1.25/\delta)},
\)
the Gaussian mechanism satisfies $(\varepsilon, \delta)$-DP.\\

\paragraph{Finite-prior exact-identification audit.}
The reconstruction threat model above still leaves open how an audit should instantiate the adversary. Prior reconstruction studies often optimize directly for a reconstructed object using attack-specific machinery, such as generative priors or training traces~\citep{balle2022informed,hayes2023bounding}. In experimental evaluations, we take a complementary route that is natural for Ball-DP (presented in the following): side information narrows the target to a small set of plausible candidates inside the declared neighborhood, the adversary's task is exact identification.

Concretely, the audit asks which candidate record would most likely have produced the observed Gaussian-perturbed model. For each candidate $z_i$, the auditor forms the model that would have been released if $z_i$ were the hidden target, compares it with the observed release, and chooses the most likely candidate under the prior. Under a uniform finite prior, this becomes a nearest-model rule: choose the candidate whose simulated trained parameter vector is closest to the observed noisy release. Section~\ref{sec:finite-prior-map-audit} formalizes this MAP rule and gives the corresponding audit algorithm.


\section{Ball-DP: Indistinguishability Over Restricted Adjacency}
\label{sec:ball-dp}

Ball-DP introduces an explicit adjacency policy for mitigating reconstruction threats. Specifically, instead of asking the privacy mechanism to hide every admissible one-record replacement, it hides one-record replacements inside a declared metric ball. The radius of this ball becomes part of the overall privacy budget along with the indistinguishability parameters.

\subsection{Adjacency policy definition}
\label{sec:r-adjacency}

Fix a public metric $d$ on the record domain $\mathcal Z$. For embeddings, $d$ is typically the Euclidean distance on the embedding coordinate, possibly with label constraints as in ~\ref{eq:label-preserving-metric}.

\begin{definition}[$r$-adjacency]
For a radius $r>0$, datasets $D = (z_1, \ldots, z_n) \in \mathcal Z^n$ and $D' = (z'_1, \ldots, z'_n) \in\mathcal Z^n$ are \emph{$r$-adjacent}, written $D\sim_r D'$, if there exists an index $j\in[n]$ such that $z_i=z_i'$ for all $i\ne j$ and
\(
d(z_j,z'_j)\le r.
\) 

\end{definition}

We now define Ball-DP as $(\varepsilon, \delta)$-indistinguishability over $r$-adjacency.

\begin{definition}[Ball-DP]
A randomized mechanism $M:\mathcal Z^n\to\Theta$ is said to satisfy \emph{$(r, \varepsilon,\delta)$-Ball-DP} if for every measurable set $S\subseteq\Theta$ and every $r$-adjacent pair $D\sim_r D'$, we have
\[
\Pr[M(D)\in S]
\le
 e^\varepsilon \Pr[M(D')\in S]+\delta.
\]
\end{definition}

The radius $r$ changes the scope of the adjacency relation being protected, not the semantics of $(\varepsilon,\delta)$ indistinguishability on that relation. For $0 \le r' \le r$, every mechanism satisfying $(r, \varepsilon,\delta)$-Ball-DP also satisfies $(r', \varepsilon,\delta)$-Ball-DP. Increasing $r$ enables protection against a larger number of substitutions, at the cost of additional noise; a trade-off made precise later in the following.

\begin{proposition}[DP as a special case]
\label{prop:bounded-replacement-special-case}
Suppose that for all $z \in \mathcal Z$, we have $\|z\|_2\le B$ and the distance metric is Euclidean: $d(z,z') \coloneqq \|z-z'\|_2$. Then $(r, \varepsilon, \delta)$-Ball-DP implies $(\varepsilon, \delta)$-DP as long as $r \ge 2B$. 
\end{proposition}

The distance metric $d(\cdot, \cdot)$ and radius $r$ makes the scope of privacy by Ball-DP explicit in the following sense. If a model release mechanism satisfies $(r, \varepsilon, \delta)$-Ball-DP then it prevents an informed adversary with access to all but the hidden record $z^*$ to estimate $z^*$ within distance $r$. Throughout this paper we consider the following label-preserving distance metric. The analyses, however, can be extended in the future to general distance metrics. \\


\paragraph{Label-preserving distance metric.}
Throughout this paper, records are pairs $(x,y)$ and we use the label-preserving Euclidean metric:
\begin{equation}
\label{eq:label-preserving-metric}
d\big((x,y),(x',y')\big)=
\begin{cases}
\|x-x'\|_2, & y=y',\\
\infty, & y\ne y'.
\end{cases}
\end{equation}
Finite-radius adjacency therefore protects embedding changes conditional on the label. It does not protect label changes; that would require a metric assigning finite distance to cross-label substitutions. This choice also makes class counts invariant over adjacent datasets, which is why count-aware calibration does not violate indistinguishability.

The metric, radius, embedding map, norm bound, and any count metadata used for calibration are treated as public or fixed by the adjacency policy. This is consistent with settings where a private learner is trained on fixed embeddings supplied by an external representation model. If any of these calibration quantities are estimated from sensitive records, that estimation must be privatized or justified as public side information before the Ball-DP claim applies.


\subsection{Ball-DP for convex learning}
\label{sec:erm}

We next connect radius-restricted adjacency to Gaussian output perturbation for the ERM problem~\eqref{eq:generic-erm}. The key quantity is how much a single record replacement can change the gradient contribution of the loss. We capture this through the following uniform regularity assumption.




\begin{assumption}[Lipschitz smoothness in sample]
\label{ass:lz}
For the parameter set $\Theta\subseteq\mathbb R^p$ containing all ERM minimizers of interest, assume that there exists a $L_z \ge 0$ such that for all $\theta\in\Theta$ and $z,z'\in\mathcal Z$,
\[
\|\nabla_\theta \ell(\theta;z)-\nabla_\theta \ell(\theta;z')\|_2
\le
L_z\,d(z,z').
\]
\end{assumption}


Section~\ref{sec:technical-statements} gives the constant $L_z$ for the ridge-prototype model used in the main experiments and records the corresponding constants for the additional convex heads. The derivations for the secondary heads are deferred to Appendix~\ref{app:proof-secondary-lz}.

\begin{theorem}[$r$-adjacency $\ell_2$-sensitivity of strongly convex ERM]
\label{thm:erm-sensitivity}
Assume that $F_D$ in~\eqref{eq:generic-erm} is differentiable and $\lambda$-strongly convex for every dataset $D$, and that Assumption~\ref{ass:lz} holds on parameter set $\Theta$. 
Then, for all $D\sim_r D'$, we have $\norm{\widehat\theta(D)-\widehat\theta(D')}_2
\le
\frac{L_z}{\lambda n}\,r$.
\end{theorem}


The proof of Theorem~\ref{thm:erm-sensitivity} can be found in Appendix~\ref{app:proof-erm-sensitivity}. The theorem shows that the distance between the corresponding ERM solutions is bounded linearly in the radius $r$.
Therefore, the Gaussian output perturbation can be calibrated to the protected radius $r$ rather than to any arbitrary replacement. 

In what follows, $\Phi$ denotes the standard normal cumulative distribution function (CDF). 
For sharper Gaussian calibration, we use the analytic Gaussian mechanism scale of~\citep{balle2018analytic}. Given an $\ell_2$ sensitivity bound $\Delta$, define
\[
\mathsf{AGM}(\varepsilon,\delta,\Delta)
:=
\inf\left\{
\sigma>0:
\Phi\left(\frac{\Delta}{2 \sigma}-\frac{\varepsilon \sigma}{\Delta}\right)-e^{\varepsilon} \Phi\left(-\frac{\Delta}{2 \sigma}-\frac{\varepsilon \sigma}{\Delta}\right) \leq \delta
\right\},
\]
where $\mathsf{AGM}(\varepsilon,\delta,0)=0$ by convention. 


\begin{theorem}[Ball-Gaussian output perturbation]
\label{thm:ball-gaussian-output-perturbation}
Fix $r > 0$, $\varepsilon>0$ and $\delta\in(0,1)$. Let $f(D)=\widehat\theta(D)$ and suppose $r$-adjacency $\ell_2$-sensitivity is bounded by $\Delta_r$, i.e., 
\[
\sup_{D\sim_r D'}\|f(D)-f(D')\|_2\le \Delta_r.
\]
If
\[
M(D)=f(D)+\xi,
\qquad
\xi\sim\mathcal N(0,\sigma^2 I_p),
\qquad
\sigma\ge \mathsf{AGM}(\varepsilon,\delta,\Delta_r),
\]
then $M$ satisfies $(r, \varepsilon,\delta)$-Ball-DP. 
\end{theorem}


\begin{proof}
For every $r$-adjacent pair $D\sim_r D'$, the two releases are Gaussian distributions with common covariance $\sigma^2I_p$ and means separated by at most $\Delta_r$. By the analytic Gaussian mechanism theorem~\citep[Theorem 8]{balle2018analytic}, the condition $\sigma\ge\mathsf{AGM}(\varepsilon,\delta,\Delta_r)$ implies
\[
\Pr[M_r(D)\in S]
\le
e^\varepsilon\Pr[M_r(D')\in S]+\delta
\]
for every measurable set $S\subseteq\Theta$. Since this holds for every $D\sim_r D'$, the mechanism satisfies $(r, \varepsilon,\delta)$-Ball-DP.
\end{proof}

Alternatively, if $\sigma\ge \frac{\Delta_r}{\varepsilon}\sqrt{2\log(1.25/\delta)}$, 
then the same conclusion follows from the classic Gaussian mechanism bound~\citep{dwork2014book}. Moreover, under Assumption~\ref{ass:lz}, due to Theorem~\ref{thm:erm-sensitivity} it suffices to use $\Delta_r=\frac{L_z}{\lambda n}r$. Note that the Gaussian mechanism calibrated to $(r, \varepsilon, \delta)$-Ball-DP also satisfies $(t, \varepsilon_t, \delta_t)$-Ball-DP for $t > r$ as long as the indistinguishability parameters $\varepsilon_t, \delta_t$ scales with the ratio $t/r$. The proof of Proposition~\ref{prop:outside-ball} is deferred to Appendix~\ref{app:proof-outside-ball}.



\begin{proposition}[Distance-dependent guarantee for Gaussian release]
\label{prop:outside-ball}
Suppose Assumption~\ref{ass:lz} holds true and the Gaussian mechanism $M$ with noise parameter $\sigma$ satisfies $(r, \varepsilon, \delta)$-Ball-DP. Then, for any $t \geq r$, $M$ satisfies $(t, \varepsilon_t, \delta_t)$-Ball-DP as long as 
\[
\delta_t
\ge
\Phi\left(\frac{\Delta_t}{2 \sigma}-\frac{\varepsilon_t \sigma}{\Delta_t}\right)-e^{\varepsilon_t} \Phi\left(-\frac{\Delta_t}{2 \sigma}-\frac{\varepsilon_t \sigma}{\Delta_t}\right),
\]
where $\Delta_t:=\frac{L_z}{\lambda n}t$.
Moreover, if the mechanism $M$ is calibrated as per the non-analytical sufficient bound
\[
\sigma\ge \frac{\Delta_r}{\varepsilon}\sqrt{2\log(1.25/\delta)},
\qquad
\Delta_r:=\frac{L_z}{\lambda n}r,
\]
then $M$ satisfies $(t, \varepsilon_t, \delta_t)$-Ball-DP with $\varepsilon':=\varepsilon \frac{t}{r}$ and $\delta_t = \delta$. 
\end{proposition}


In particular, if a mechanism has been calibrated for radius $r$, and one later evaluates substitutions at a larger radius $t>r$, the same Gaussian release still yields a weaker but valid accounting guarantee, with privacy parameter inflated from $\varepsilon$ to $\varepsilon \frac{t}{r}$.
Hence, when $\norm{z}_2 \leq B$ for all records $z$, combining Proposition~\ref{prop:bounded-replacement-special-case} with Proposition~\ref{prop:outside-ball}, the Ball-DP privacy guarantee can be converted into standard DP guarantee by setting $t = 2B$.

\section{Ball-ReRo Certificates and Audits}
\label{sec:gaussian-rero}

We now translate the radius-restricted indistinguishability guarantee into reconstruction-robustness statements. The goal is not to define a new attack model, but to make precise what an informed adversary can gain from the released Gaussian-perturbed model beyond its prior information. The section has three parts. First, we state the exact finite-prior MAP audit used in the experiments. Second, we define the Ball-ReRo baseline and show how Ball-DP and Ball-RDP imply reconstruction certificates. Third, we give the direct Gaussian Ball-ReRo bound used as the empirical certificate.

Throughout this section, the adversary knows the public remainder dataset $D^-$, the training algorithm, the hyperparameters, the Ball-DP policy radius, and the candidate-construction protocol. The hidden record is denoted by $Z$, the completed dataset is $D^-\cup\{Z\}$, and the released model is
\[
\widetilde\theta
=
\widehat\theta(D^-\cup\{Z\})+
\xi,
\qquad
\xi\sim\mathcal N(0,\sigma^2I).
\]
The audit asks whether $\widetilde\theta$ identifies the hidden record better than what is already possible from the prior and the allowed candidate set.

\subsection{Finite-prior MAP audit for Gaussian releases}
\label{sec:finite-prior-map-audit}

We consider a finite candidate set
\[
S=\{z_1,\dots,z_m\}\subseteq \ball{d}{u,r}
\]
that contains the hidden target, with prior weights $\pi_1,\dots,\pi_m>0$ satisfying $\sum_{i=1}^m\pi_i=1$. This is an adversary-favorable audit: the side information has already reduced the reconstruction problem to choosing among plausible local alternatives. The remaining question is whether the released model identifies which candidate was used as the hidden training record.

\begin{proposition}[Exact finite-prior MAP attack for Gaussian release]
\label{prop:finite-map}
For the Gaussian release described above, a Bayes-optimal exact-identification attack under the finite prior
\[
\pi_S=\sum_{i=1}^m \pi_i\delta_{z_i}
\]
is
\[
\widehat z_{\mathrm{MAP}}
\in
\arg\max_{z_i\in S}
\left\{
\log \pi_i
-
\frac{1}{2\sigma^2}
\bigl\|
\widetilde\theta-\widehat\theta(D^-\cup\{z_i\})
\bigr\|_2^2
\right\}.
\]
\end{proposition}

For a uniform prior on $S$, the $\log\pi_i$ term is constant and the attack reduces to constrained least squares over the candidate set. The proof is deferred to Appendix~\ref{app:proof-finite-map}.

\begin{algorithm}[t]
\caption{Finite-prior MAP audit for Gaussian release}
\label{alg:finite-map}
\begin{algorithmic}[1]
\Require observed release $\widetilde\theta$, public remainder $D^-$, noise scale $\sigma$
\Require candidate set $S=\{z_1,\dots,z_m\}\subseteq \ball{d}{u,r}$ and prior weights $\pi_1,\dots,\pi_m$
\For{$i=1,\dots,m$}
    \State $\theta_i \gets \widehat\theta(D^-\cup\{z_i\})$ \Comment{candidate model under target $z_i$}
    \State $s_i \gets \log \pi_i - \dfrac{1}{2\sigma^2}\|\widetilde\theta-\theta_i\|_2^2$ \Comment{Gaussian log-posterior up to constants}
\EndFor
\State \Return $z_{\hat\imath}$ for any $\hat\imath\in\arg\max_i s_i$
\end{algorithmic}
\end{algorithm}

\subsection{Ball-supported priors and feasible outputs}

A reconstruction certificate must be interpreted relative to the adversary's side information. In particular, if the audit protocol restricts the reconstructor to a feasible output set, then the oblivious baseline should obey the same restriction. We therefore define the prior support, the allowed-output baseline, and the Ball-ReRo property using the same feasible-output family.

\begin{definition}[Ball-supported prior]
A family of priors $\{\pi_{u,r}\}_{u\in\mathcal Z}$ is \emph{Ball-supported} if each $\pi_{u,r}$ is supported on the policy ball
\[
\ball{d}{u,r}:=\{z\in\mathcal Z:d(z,u)\le r\}.
\]
The record $u$ is a public anchor or reference record and need not itself be sampled from the prior.
\end{definition}

\begin{definition}[Allowed-output reconstruction baseline]
\label{def:kappa}
Fix a prior $\pi$, reconstruction loss $\rho:\mathcal Z\times\mathcal Z\to[0,\infty)$, tolerance $\eta\ge0$, and feasible output set $\mathcal A\subseteq\mathcal Z$. Define
\[
\kappa_{\pi,\rho}(\eta;\mathcal A)
:=
\sup_{a\in\mathcal A}
\Pr_{Z\sim\pi}[\rho(Z,a)\le\eta].
\]
This is the best success probability of an oblivious adversary that ignores the mechanism output while obeying the same feasible-output restriction as the reconstructor. When the prior, loss, tolerance, and feasible set are clear from context, we write this baseline simply as $\kappa$.
\end{definition}

\begin{definition}[Ball-ReRo]
\label{def:ball-rero}
Fix a radius $r>0$, a reconstruction loss $\rho$, a tolerance $\eta\ge0$, a Ball-supported prior family $\{\pi_{u,r}\}_{u\in\mathcal Z}$, and feasible output sets $\{\mathcal A_{u,r}\}_{u\in\mathcal Z}$ with $\mathcal A_{u,r}\subseteq\mathcal Z$. A mechanism $M:\mathcal Z^n\to\Theta$ satisfies $(\eta,\gamma,r)$-Ball-ReRo with respect to $(\rho,\{\pi_{u,r}\},\{\mathcal A_{u,r}\})$ if, for every public remainder dataset $D^-\in\mathcal Z^{n-1}$, every anchor $u\in\mathcal Z$, and every reconstructor $R:\Theta\times\mathcal Z^{n-1}\to\mathcal A_{u,r}$,
\[
\Pr_{Z\sim\pi_{u,r},\,\theta\sim M(D^-\cup\{Z\})}
\bigl[
\rho(Z,R(\theta,D^-))\le\eta
\bigr]
\le \gamma.
\]
\end{definition}

For finite exact-identification audits, the reconstructor is constrained to output one of the candidates in $S$. With a uniform prior on $m$ candidates and exact-identification loss, the baseline in Definition~\ref{def:kappa} is therefore $\kappa=1/m$. A smaller policy radius may make it harder to construct a large, well-separated feasible set near a given anchor; this is part of the operational trade-off and must be reported explicitly with the audit.

\begin{theorem}[Ball-DP implies Ball-ReRo]
\label{thm:ball-dp-rero}
Fix $D^-\in\mathcal Z^{n-1}$, an anchor $u\in\mathcal Z$, a Ball-supported prior $\pi_{u,r}$, a loss $\rho$, a tolerance $\eta$, and a feasible output set $\mathcal A_{u,r}\subseteq\mathcal Z$. Let
\[
\kappa:=\kappa_{\pi_{u,r},\rho}(\eta;\mathcal A_{u,r}).
\]
If $M$ satisfies $(r,\varepsilon,\delta)$-Ball-DP, then every reconstructor $R:\Theta\times\mathcal Z^{n-1}\to\mathcal A_{u,r}$ satisfies
\[
\Pr_{Z\sim\pi_{u,r},\,\theta\sim M(D^-\cup\{Z\})}
\bigl[
\rho(Z,R(\theta,D^-))\le\eta
\bigr]
\le e^{\varepsilon}\kappa+\delta.
\]
Equivalently, $M$ satisfies $(\eta,e^\varepsilon\kappa+\delta,r)$-Ball-ReRo for this prior and feasible-output set.
\end{theorem}

The proof is in Appendix~\ref{app:proof-ball-dp-rero}. Theorem~\ref{thm:ball-dp-rero} is the radius-restricted analogue of the standard observation that DP limits reconstruction attacks~\citep{balle2022informed}. The important operational term is the baseline $\kappa$: it is the best success probability achievable without observing the mechanism output, under the same prior and feasible-output restriction as the reconstructor. If the prior and side information already make this baseline large, then a high raw reconstruction rate may mostly reflect the audit setup rather than additional information leaked by the release. The relevant empirical question is therefore how far the observed success probability rises above $\kappa$.

We next adapt Rényi-DP (RDP)~\citep{mironov2017renyi} to Ball-DP, which we call {\em Ball-RDP}, and state the corresponding Ball-ReRo guarantee.

\begin{definition}[Ball-RDP]
A mechanism $M$ satisfies \emph{$(\alpha,\varepsilon_\alpha)$ Ball-RDP at radius $r$} if for every $D\sim_r D'$,
\[
D_\alpha\big(M(D)\,\|\,M(D')\big)\le \varepsilon_\alpha,
\qquad \alpha>1.
\]
\end{definition}

The Gaussian release mechanism can be adjusted as follows to ensure Ball-RDP.

\begin{lemma}[Gaussian Ball-RDP]
\label{lem:gaussian-ball-rdp}
Let $M(D)=f(D)+\xi$ with $\xi\sim\mathcal N(0,\sigma^2I)$ and $r$-adjacency sensitivity
\[
\Delta_2(r):=\sup_{D\sim_r D'}\|f(D)-f(D')\|_2.
\]
Then for every order $\alpha>1$,
\[
D_\alpha\big(M(D)\,\|\,M(D')\big)
\le
\frac{\alpha\Delta_2(r)^2}{2\sigma^2}
\qquad\text{for all }D\sim_r D'.
\]
\end{lemma}

Lastly, we obtain the following Ball-ReRo guarantee for the Gaussian Ball-RDP mechanism.

\begin{theorem}[Ball-RDP implies Ball-ReRo]
\label{thm:ball-rdp-rero}
Fix $D^-\in\mathcal Z^{n-1}$, a Ball-supported prior $\pi_{u,r}$, a loss $\rho$, a tolerance $\eta$, and a feasible output set $\mathcal A_{u,r}\subseteq\mathcal Z$. Let $\kappa:=\kappa_{\pi_{u,r},\rho}(\eta;\mathcal A_{u,r})$. If a mechanism satisfies $(\alpha,\varepsilon_\alpha)$ Ball-RDP at radius $r$ for some $\alpha>1$, then every reconstructor $R$ with outputs in $\mathcal A_{u,r}$ obeys the Ball-ReRo inequality with
\[
\gamma = \gamma_{\alpha}
:=
\min\left\{1,
\kappa^{(\alpha-1)/\alpha}
\exp\left(\frac{\alpha-1}{\alpha}\varepsilon_\alpha\right)
\right\}.
\]
\end{theorem}

The proof of Lemma~\ref{lem:gaussian-ball-rdp} is in Appendix~\ref{app:proof-gaussian-ball-rdp}, and the proof of Theorem~\ref{thm:ball-rdp-rero} is in Appendix~\ref{app:proof-ball-rdp-rero}.

\subsection{Gaussian Ball-ReRo bounds}

The preceding theorems apply to Ball-DP and Ball-RDP mechanisms, but they can be loose for Gaussian releases. We therefore use the testing, or blow-up-function, viewpoint on reconstruction robustness~\citep{balle2022informed,kaissis2023rero,hayes2023bounding}. For distributions $P$ and $Q$ on the release space and $\kappa\in[0,1]$, write
\[
\mathcal B_\kappa(P,Q)
:=
\sup\{P(E): E\text{ measurable and }Q(E)\le \kappa\}.
\]
This quantity asks how large an attack-success event can be under the target-dependent release law if the same event has baseline probability at most $\kappa$ under a reference release law. In our Gaussian output-perturbation setting, the relevant output laws have equal covariance, so this testing problem has a closed form; Appendix~\ref{app:gaussian-testing-lemma} proves the Gaussian identity used below. The next results are the certificates plotted in the experiments and should be read as release-level upper bounds under the stated feasible-set prior.

\begin{theorem}[Direct Gaussian Ball-ReRo]
\label{thm:direct-gaussian-rero}
Let $M(D)=f(D)+\xi$ with $\xi\sim\mathcal N(0,\sigma^2I)$, and let $\Delta_2(r)$ denote the $r$-adjacency sensitivity of $f$. Fix $D^-\in\mathcal Z^{n-1}$, an anchor $u\in\mathcal Z$, a Ball-supported prior $\pi_{u,r}$, a reconstruction loss $\rho$, a tolerance $\eta$, and a feasible output set $\mathcal A_{u,r}\subseteq\mathcal Z$. Define
\[
\kappa:=\kappa_{\pi_{u,r},\rho}(\eta;\mathcal A_{u,r})
\]
and
\[
\bar\Delta(D^-)
:=
\sup_{z\in\operatorname{supp}(\pi_{u,r})}
\|f(D^-\cup\{z\})-f(D^-\cup\{u\})\|_2.
\]
Then every reconstructor $R$ with outputs in $\mathcal A_{u,r}$ satisfies the following
\[
\begin{aligned}
\Pr[\rho(Z,R(\theta,D^-))\le\eta]
&\le
\Phi\!\left(
\Phi^{-1}(\kappa)+
\frac{\bar\Delta(D^-)}{\sigma}
\right) \\
&\le
\Phi\!\left(
\Phi^{-1}(\kappa)+
\frac{\Delta_2(r)}{\sigma}
\right),
\end{aligned}
\]
Here the displayed bound is extended continuously at the endpoint baselines: if $\kappa=0$ the bound is $0$, and if $\kappa=1$ the bound is $1$. The proof is in Appendix~\ref{app:proof-direct-gaussian-rero}.

\end{theorem}

\begin{corollary}[Uniform finite prior under exact identification]
\label{cor:uniform-finite-prior}
Suppose $\pi_S$ is the uniform prior on $m$ candidates,
\[
\pi_S=\frac{1}{m}\sum_{i=1}^m\delta_{z_i},
\qquad
S\subseteq\ball{d}{u,r},
\]
and the decoder is constrained to output a candidate in $S$. Consider the exact-identification loss $\rho_{\mathrm{ID}}(z,z'):=\ind{z\ne z'}$ with tolerance $\eta=0$. Then, every reconstructor $R$ satisfies
\[
\Pr[R(\theta,D^-)=Z]
\le
\Phi\!\left(
\Phi^{-1}(1/m)+\frac{\Delta_2(r)}{\sigma}
\right),
\]
where the probability is over $Z\sim\pi_S$ and $\theta\sim M(D^-\cup\{Z\})$. The feasible-output oblivious baseline is exactly $\kappa=1/m$.
\end{corollary}

The proof is in Appendix~\ref{app:proof-uniform-finite-prior}. The direct bound is the main certificate plotted in the experiments. Under homogeneous Gaussian calibration at fixed $(\varepsilon,\delta)$, the ratio $\Delta_2(r)/\sigma$ is fixed by the calibration rule, so shrinking $r$ does not by itself improve the numerical Ball-ReRo bound. The radius still matters because it changes both the protected adjacency scope and the amount of noise needed to reach that certificate. The empirical MAP audit then measures the resulting identification trade-off for the feasible candidate sets that can actually be constructed inside the policy ball. A sharper support-specific Gaussian exact-identification bound is available, but it is not used in the reported experiments; we state it in Appendix~\ref{app:support-specific-bound}.

\section{Experimental Analysis}
\label{sec:experiments}

\subsection{Goals and scope}
The empirical evaluation asks three operational questions. First, at the same target $(\varepsilon,\delta)$, how much accuracy is recovered by calibrating to a declared local radius rather than to the global bounded-replacement radius? Second, for a fixed reconstruction certificate $\gamma$, how much utility is lost as the protected local neighborhood grows? Third, when an informed adversary is restricted to a finite local candidate set, how do empirical MAP attack rates compare with the oblivious baseline and the theorem-backed certificate?

All reported ERM results use $\delta=10^{-6}$, $m=10$, and Gaussian output perturbation. We choose the Ball-DP radius $r$ through empirical local-radius policies. For each benchmark, $q50$, $q80$, and $q95$ denote the radii obtained from the $50$th, $80$th, and $95$th percentiles of the same-label distance distribution under the label-preserving metric~\eqref{eq:label-preserving-metric}. These three policies represent progressively larger protected neighborhoods; Sections~\ref{sec:utility-results} and~\ref{sec:finite-prior-diagnostics} describe how they are used in the accuracy and MAP audits. The standard comparator uses the full bounded-replacement radius $r_{\mathrm{std}}=2B$ for utility experiments and an empirical same-label diameter in the finite-prior attack audit.

\subsection{Experimental model}
The seven-dataset utility table and the finite-prior MAP audit use a ridge-prototype classifier on fixed embeddings. This head is the closed-form audit case: its optimizer, sensitivity, and noiseless informed-adversary inverse problem are explicit, which lets us separate privacy calibration from numerical approximation error. The Ball-DP calibration theorem itself is not ridge-specific; Section~\ref{sec:secondary-head-statements} verifies the required $L_z$ constants for binary logistic regression, softmax logistic regression, and squared-hinge SVMs, and Section~\ref{sec:convex-head-imdb} reports an IMDb utility--certificate sweep for two of these additional heads alongside ridge prototypes.

For record $(z,y)$ with $z\in\mathbb R^d$ and $y\in\{1,\ldots,C\}$, the ridge-prototype model learns one prototype $\mu_c\in\mathbb R^d$ per class by minimizing
\begin{equation}
\label{eq:ridge-proto-objective}
F_D(\mu)=\frac{1}{n}\sum_{i=1}^n\|z_i-\mu_{y_i}\|_2^2+\frac{\lambda}{2}\sum_{c=1}^C\|\mu_c\|_2^2.
\end{equation}
Prediction uses the nearest prototype. The full ridge-prototype statements are collected in Section~\ref{sec:technical-statements}, with the core learning proofs in Section~\ref{sec:core-learning-proofs}. In the reported runs $\lambda=0.01$.

For this model, Corollary~\ref{cor:proto-sensitivity} gives the count-worst-case exact local sensitivity
\[
\Delta_{\mathrm{proto}}(r)=\frac{2r}{2+\lambda n}.
\]
When class counts are public, the count-aware variant is stated separately in Proposition~\ref{prop:proto-count-aware-sensitivity}. Gaussian output perturbation is then calibrated by substituting the relevant sensitivity into Theorem~\ref{thm:ball-gaussian-output-perturbation}.

\subsection{Utility at local policy radii}
\label{sec:utility-results}

Table~\ref{tab:accuracy-radius-summary} reports the main utility comparison at $\varepsilon=2$ across seven embedding benchmarks. It replaces the earlier single-radius noise-ratio table because it answers the operational question directly: what accuracy do stakeholders get when they choose a protected neighborhood? At this privacy level, all tabulated Ball-DP radii improve over the global-radius comparator. Smaller local radii generally yield higher utility, with near ties at the reported precision across several datasets; $q95$ provides the broadest local protection among the three local policies.

\begin{table}[t]
\centering
\caption{\small Mean accuracies over ten independent Gaussian release seeds for the ridge-prototype private ERM experiment at $\varepsilon=2$ and $m=10$ across seven embedding benchmarks. Ball-DP columns use local adjacency radii $q50$, $q80$, and $q95$; Std is the global-radius comparator calibrated with the same Gaussian rule.}
\label{tab:accuracy-radius-summary}
\scriptsize
\setlength{\tabcolsep}{4pt}
\begin{tabular}{lrrrr}
\toprule
Data embeddings & Acc $q50$ & Acc $q80$ & Acc $q95$ & Acc (Std) \\
\midrule
AG News & 0.8616 & 0.8610 & 0.8606 & 0.8559 \\
BANKING77 & 0.6358 & 0.5697 & 0.5218 & 0.2711 \\
CIFAR-10 & 0.7374 & 0.7360 & 0.7340 & 0.6649 \\
Emotion & 0.3010 & 0.2932 & 0.2887 & 0.2458 \\
IMDb & 0.6212 & 0.6135 & 0.6071 & 0.5632 \\
MNIST & 0.8356 & 0.8345 & 0.8330 & 0.7840 \\
TREC-6 & 0.2496 & 0.2460 & 0.2412 & 0.2036 \\
\bottomrule
\end{tabular}
\end{table}

For text datasets, the fixed embeddings are computed with the \texttt{all-MiniLM-L6-v2} Sentence Transformers model, which is based on Sentence-BERT and MiniLM \cite{reimers2019sentencebert,wang2020minilm,sentenceTransformersAllMiniLML6V2}. For image datasets, the fixed embeddings are computed with ImageNet-pretrained ResNet-18 features implemented in TorchVision \cite{he2016resnet,paszke2019pytorch,torchvisionResNet18}. The embedding maps are treated as public preprocessing components of the adjacency policy. The benchmark suite consists of AG News~\cite{zhang2015character}, BANKING77~\cite{casanueva2020efficient}, CIFAR-10~\cite{krizhevsky2009learning}, Emotion~\cite{saravia2018carer}, IMDb~\cite{maas2011learning}, MNIST~\cite{lecun1998gradient}, and TREC-6~\cite{li2002learning}.

The absolute accuracies should not be read as benchmark-leading numbers: the ridge-prototype head is intentionally simple and can be poorly matched to some fixed embeddings, notably on Emotion and TREC-6. We retain these rows as stress tests for the privacy calibration rather than as claims of competitive task performance.

The gains are not uniform, which is itself useful for deployment planning. At $q50$, improvements over the standard comparator range from $0.03$ percentage points on AG News to $14.00$ points on BANKING77. At $q80$, the range is $0.04$--$12.37$ points; at $q95$, it is $0.03$--$11.09$ points. These differences quantify the utility cost of protecting a larger local neighborhood at the same $(\varepsilon,\delta)$.

\subsection{Accuracy at a fixed reconstruction certificate}
\label{sec:accuracy-vs-gamma}

The most interpretable privacy--utility visualization is accuracy as a function of the certified exact-identification upper bound $\gamma$ from Corollary~\ref{cor:uniform-finite-prior}. A higher point means better utility; a farther-left point means a smaller certified exact-identification risk. Figure~\ref{fig:accuracy-vs-gamma-radius-grid-main} shows the regenerated seven-dataset private-ERM accuracy--certificate trade-off from the publication report. The plot uses the finite rows for the direct certificate column and aggregates seed-first before plotting.

\begin{figure}[htb!]
\centering
\includegraphics[width=\linewidth]{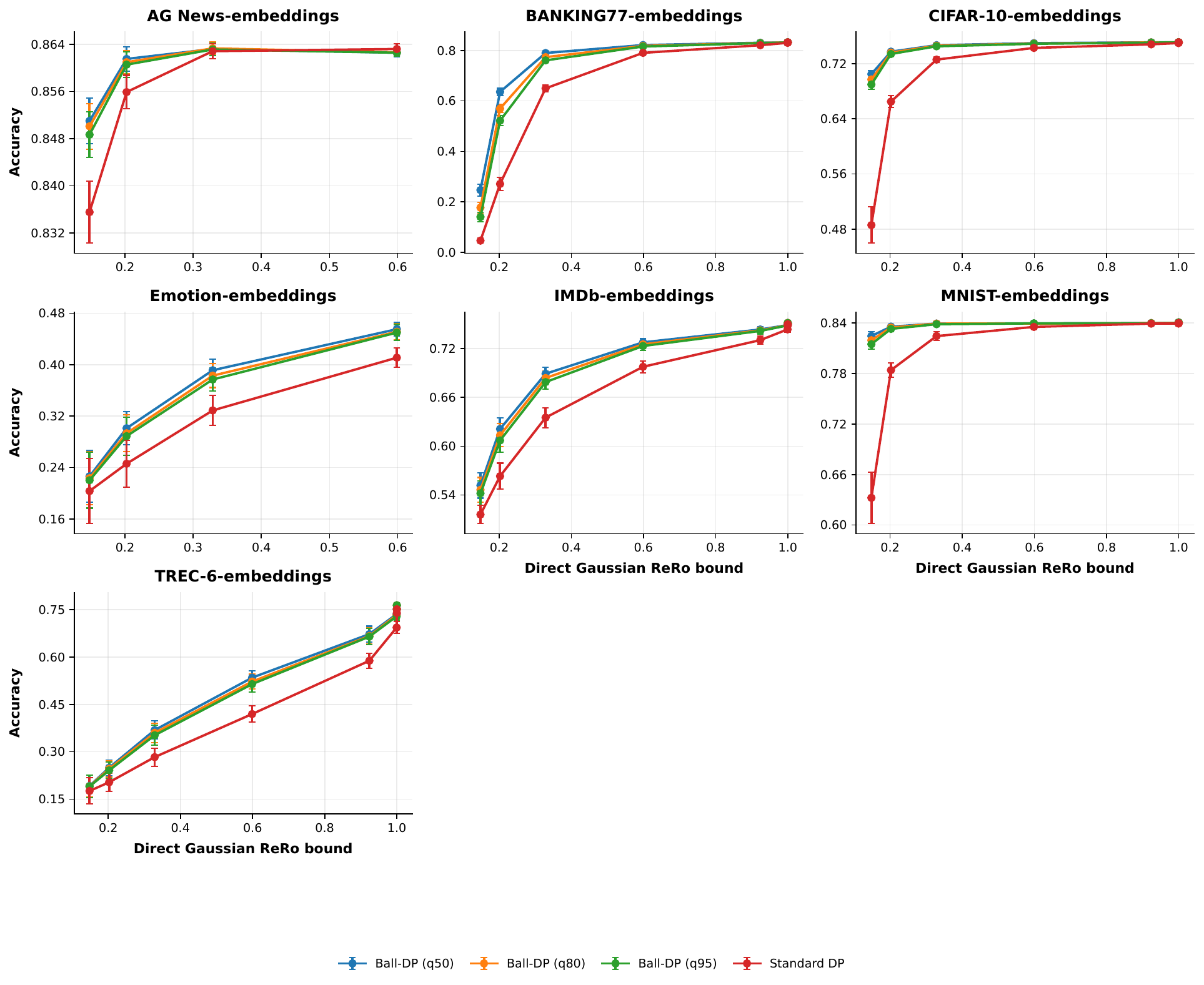}
\caption{\small Private-ERM accuracy versus the direct exact-identification upper bound $\gamma$ on the seven embedding benchmarks for the ridge-prototype head. Solid curves are Ball-DP mechanisms calibrated at local radii $q50$, $q80$, and $q95$; the dashed curve is the global-radius comparator. Points are seed-first means over ten independent Gaussian release seeds. At a fixed numerical certificate level, the vertical gap shows the utility gained by protecting a local neighborhood rather than the global bounded-replacement radius. Table~\ref{tab:accuracy-radius-summary} gives the corresponding per-dataset accuracy values.}
\label{fig:accuracy-vs-gamma-radius-grid-main}
\end{figure}

The important interpretation is that the radius labels do not denote different certificate values at the same $\varepsilon$. Under matched Gaussian calibration, $\gamma$ is determined by $m$, $\varepsilon$, and $\delta$ through $\Delta_2/\sigma$. The radius instead tells the reader which substitutions are protected and how much utility is retained while achieving that numerical certificate.

\subsection{Convex-head comparison on IMDb}
\label{sec:convex-head-imdb}

The preceding seven-dataset frontier uses ridge prototypes to keep the exact finite-prior inverse problem explicit. The calibration theorem, however, applies to every convex head for which Assumption~\ref{ass:lz} is verified. Proposition~\ref{prop:secondary-lz-summary} gives those constants for binary logistic regression, softmax logistic regression, and squared-hinge SVMs. To make this scope visible empirically, Figure~\ref{fig:imdb-convex-heads-rero} repeats the private-ERM utility--certificate sweep on IMDb embeddings at the local radius policy $q50$ for binary logistic regression, ridge prototypes, and the squared-hinge SVM. This experiment uses five Gaussian release seeds per mechanism and nominal privacy levels $\varepsilon\in\{8,16,32\}$.

\begin{figure}[htb!]
\centering
\includegraphics[width=\linewidth]{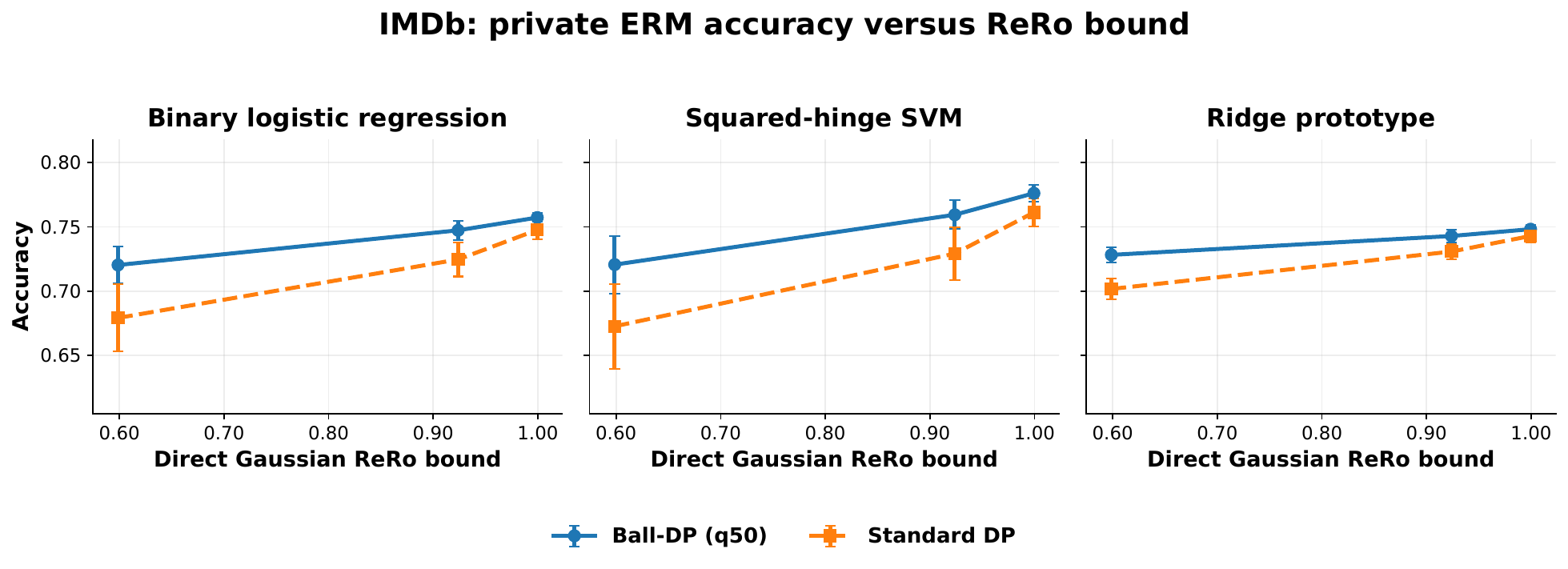}
\caption{\small IMDb convex-head comparison at local radius policy $q50$. Each panel shows private-ERM accuracy as a function of the direct finite-prior ReRo certificate $\gamma$. Solid curves are Ball-DP releases calibrated at the local $q50$ policy radius; dashed curves are the global-radius Standard-DP comparator evaluated on the corresponding $q50$ row set. Points correspond to $\varepsilon\in\{8,16,32\}$, with $\gamma=0.5987,0.9238,0.9995$, respectively. Error bars are seed-first $95\%$ normal intervals over five Gaussian release seeds.}
\label{fig:imdb-convex-heads-rero}
\end{figure}

\begin{table}[htb!]
\centering
\caption{\small IMDb convex-head comparison at local radius policy $q50$. Each cell reports mean accuracy $\pm$ seed-first $95\%$ confidence interval over five Gaussian release seeds. The ReRo certificates are $\gamma=0.5987,0.9238,0.9995$ for $\varepsilon=8,16,32$, respectively, and are the horizontal-axis values used in Figure~\ref{fig:imdb-convex-heads-rero}.}
\label{tab:imdb-convex-heads-rero}
\scriptsize
\setlength{\tabcolsep}{4pt}
\begin{tabular}{@{}llccc@{}}
\toprule
Head & Mechanism & $\varepsilon=8$ & $\varepsilon=16$ & $\varepsilon=32$ \\
\midrule
Binary logistic & Ball-DP $(q50)$
& $0.720 \pm 0.014$ & $0.747 \pm 0.008$ & $0.757 \pm 0.004$ \\
Binary logistic & Standard DP
& $0.679 \pm 0.026$ & $0.725 \pm 0.013$ & $0.748 \pm 0.007$ \\
\midrule
Ridge prototype & Ball-DP $(q50)$
& $0.728 \pm 0.006$ & $0.743 \pm 0.005$ & $0.748 \pm 0.003$ \\
Ridge prototype & Standard DP
& $0.702 \pm 0.008$ & $0.731 \pm 0.006$ & $0.743 \pm 0.005$ \\
\midrule
Squared-hinge SVM & Ball-DP $(q50)$
& $0.721 \pm 0.023$ & $0.759 \pm 0.011$ & $0.776 \pm 0.006$ \\
Squared-hinge SVM & Standard DP
& $0.673 \pm 0.033$ & $0.729 \pm 0.020$ & $0.761 \pm 0.011$ \\
\bottomrule
\end{tabular}
\end{table}

At every displayed privacy level and for all three convex heads, Ball-DP has higher mean accuracy than the global-radius comparator at the same reported ReRo certificate. The gains are largest at the strongest privacy level, $\varepsilon=8$: $4.12$ accuracy points for binary logistic regression, $2.65$ points for ridge prototypes, and $4.80$ points for the squared-hinge SVM. As privacy weakens, the gap shrinks: at $\varepsilon=32$, the corresponding gains are $0.96$, $0.53$, and $1.49$ points. This pattern is consistent with the calibration story: local radius-aware noise matters most when the Gaussian perturbation is large.

The head comparison also shows a utility trade-off that is separate from the privacy certificate. When one can afford high $\varepsilon$, the secondary convex heads can outperform the ridge prototype: at $\varepsilon=32$, the squared-hinge SVM reaches the highest mean accuracy in this sweep, $0.776$ under Ball-DP. However, these higher-utility heads are also more sensitive to stronger privacy in the reported experiment. From $\varepsilon=32$ to $\varepsilon=8$, Ball-DP accuracy drops by $3.68$ points for binary logistic regression and $5.54$ points for the squared-hinge SVM, compared with $2.00$ points for the ridge prototype. Thus the secondary heads can be attractive in the high-utility, weak-privacy regime, while the ridge prototype remains a stable closed-form audit head for the main seven-dataset reconstruction study.

\subsection{Finite-prior reconstruction diagnostics}
\label{sec:finite-prior-diagnostics}

This diagnostic is not a new privacy definition; it is an attack audit for a concrete finite candidate support. The prior on that support is part of the attack specification. In the reported audit we use the uniform prior over $m=10$ candidates, so the oblivious exact-identification baseline is exactly $1/m=0.1$ and comparisons are not confounded by prior imbalance. Proposition~\ref{prop:finite-map} also covers non-uniform initial priors: changing the prior adds the term $\log \pi_i$ to the MAP score and changes the exact-identification baseline to $\max_i \pi_i$. Such priors are useful for deployment-specific stress tests, but the aggregate table below uses the uniform prior throughout.

Within this uniform-prior protocol, the radius-grid audit varies $\varepsilon\in\{2,4,8\}$ and the Ball-DP radius policy in $\{q50,q80,q95\}$, while evaluating the Standard-DP comparator on the same radius-$q80$ feasible-set protocol. For each setup seed, we construct four public-only same-label feasible supports inside the policy ball. Each support has size $m=10$ and is chosen by a farthest-point procedure to stress-test identification among well-separated feasible candidates. We then cycle through the support: for each candidate $z_i\in S$, we treat $z_i$ once as the hidden target, release a Gaussian-perturbed ridge-prototype model trained on $D^-\cup\{z_i\}$, and run the exact MAP rule of Proposition~\ref{prop:finite-map}. Thus each setup seed contributes $4\times 10=40$ raw release-and-attack trials, and the ten setup seeds give $400$ raw trials per dataset and mechanism. The reported uncertainty is computed seed-first: raw trials are first averaged within each setup seed, and the displayed confidence intervals are then computed across the ten setup-seed averages. The standard mechanism in this audit uses the empirical same-label diameter and count-aware ridge calibration.

\begin{figure}[htb!]
\centering
\includegraphics[width=\linewidth]{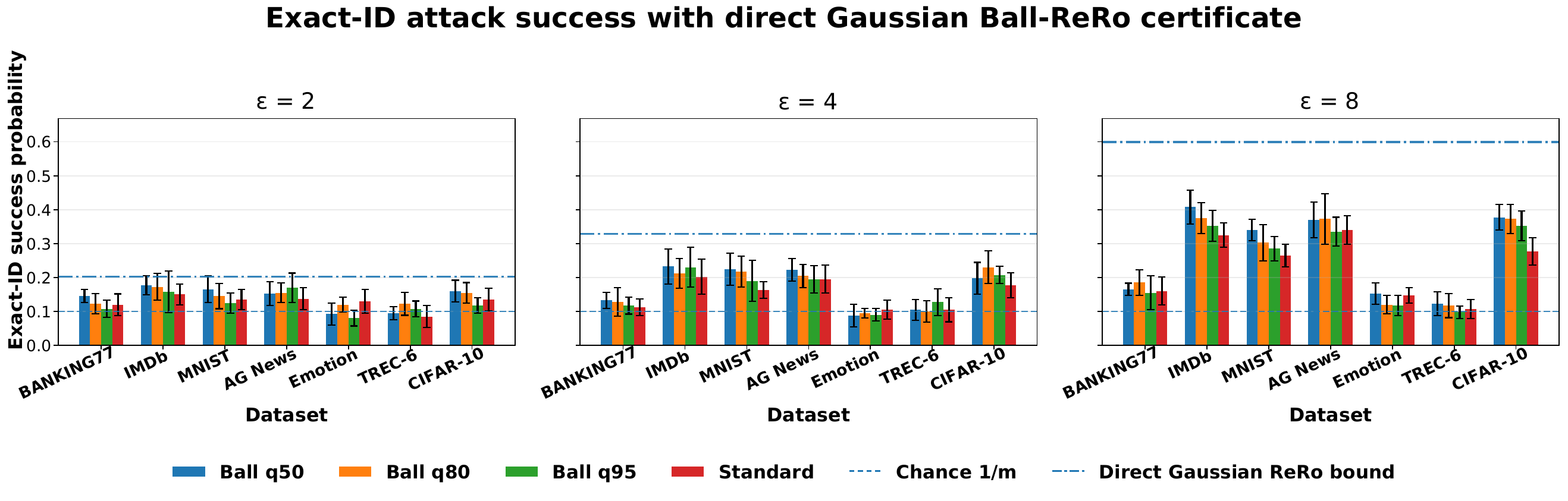}
\caption[Radius-grid exact-identification audit]{\small Radius-grid exact-identification audit for ridge-prototype finite-prior MAP reconstruction. Each panel fixes $\varepsilon\in\{2,4,8\}$; within each dataset, the first three bars show Ball-DP calibrated at radius policies $q50$, $q80$, and $q95$, while the fourth bar shows the global-radius Standard-DP Gaussian comparator evaluated on the same radius-$q80$ finite-prior feasible-set protocol. The $q80$ label for the Standard-DP bar therefore describes the attack candidate-support construction, not the Standard-DP privacy adjacency radius. Error bars are seed-first $95\%$ confidence intervals over ten setup seeds, where each setup seed averages four feasible supports and all $m=10$ candidate-target trials within each support. The dashed horizontal line is chance, $1/m=0.1$. The upper horizontal line is the direct Gaussian ReRo certificate $\gamma_{\mathrm{rel}}=\Phi(\Phi^{-1}(1/m)+\Delta/\sigma)$, which is independent of the radius policy within a fixed $\varepsilon$ panel. For $\delta=10^{-6}$ and $m=10$, the displayed certificate values are $0.2024$, $0.3286$, and $0.5987$ for $\varepsilon=2,4,8$, respectively.}
\label{fig:exact-id-radius-grid}
\end{figure}

The radius-grid audit in Figure~\ref{fig:exact-id-radius-grid} and Table~\ref{tab:exact-id-radius-grid-summary} should not be read as saying that Ball-DP and the standard comparator protect the same adjacency relation. They do not. Ball-DP uses less noise because it protects the declared local ball, whereas the standard comparator uses the global replacement radius. The empirical exact-ID rate therefore need not be lower for Ball-DP and can be higher depending on the sampled feasible-set geometry. The ReRo reference line gives the corresponding direct Gaussian worst-case certificate
\[
\gamma_{\mathrm{rel}}
=
\Phi\!\left(\Phi^{-1}(1/m)+\Delta/\sigma\right),
\]
with the calibration values reported in Table~\ref{tab:direct-gaussian-rero-bounds}. Across the reruns shown here, the empirical success probabilities remain below the certificate lines while staying interpretable relative to the chance baseline $1/m$.

\begin{table*}[htb!]
\centering
\caption{\small Exact-identification radius-grid audit for ridge-prototype finite-prior MAP reconstruction. Entries report empirical exact-ID success probability with seed-first $95\%$ confidence intervals over ten setup seeds. Each displayed empirical cell uses $400$ raw release-and-attack trials, decomposed as ten setup seeds, four feasible supports per setup seed, and all $m=10$ candidate-target trials within each support. The first three columns vary the Ball-DP radius policy. The Standard-DP column uses the global-radius Gaussian comparator evaluated on the same radius-$q80$ finite-prior feasible-set protocol; this matches Figure~\ref{fig:exact-id-radius-grid}. Chance is $1/m=0.1$ for all rows. The final column gives the direct Gaussian ReRo certificate $\gamma_{\mathrm{rel}}=\Phi(\Phi^{-1}(1/m)+\Delta/\sigma)$.}
\label{tab:exact-id-radius-grid-summary}
\resizebox{\textwidth}{!}{%
\begin{tabular}{llccccc}
\toprule
Dataset & $\varepsilon$ & Ball-DP $q50$ & Ball-DP $q80$ & Ball-DP $q95$ & Standard DP $(q80\ \mathrm{audit})$ & $\gamma_{\mathrm{rel}}$ \\
\midrule
AG News & $2$ & $0.153 \pm 0.035$ & $0.155 \pm 0.029$ & $0.170 \pm 0.043$ & $0.138 \pm 0.032$ & $0.2024$ \\
AG News & $4$ & $0.223 \pm 0.034$ & $0.205 \pm 0.034$ & $0.195 \pm 0.041$ & $0.195 \pm 0.041$ & $0.3286$ \\
AG News & $8$ & $0.370 \pm 0.052$ & $0.373 \pm 0.074$ & $0.335 \pm 0.043$ & $0.340 \pm 0.041$ & $0.5987$ \\
\midrule
BANKING77 & $2$ & $0.145 \pm 0.019$ & $0.123 \pm 0.030$ & $0.108 \pm 0.025$ & $0.120 \pm 0.032$ & $0.2024$ \\
BANKING77 & $4$ & $0.133 \pm 0.023$ & $0.128 \pm 0.042$ & $0.118 \pm 0.025$ & $0.113 \pm 0.025$ & $0.3286$ \\
BANKING77 & $8$ & $0.165 \pm 0.018$ & $0.185 \pm 0.037$ & $0.155 \pm 0.051$ & $0.160 \pm 0.041$ & $0.5987$ \\
\midrule
CIFAR-10 & $2$ & $0.160 \pm 0.032$ & $0.155 \pm 0.030$ & $0.118 \pm 0.023$ & $0.135 \pm 0.034$ & $0.2024$ \\
CIFAR-10 & $4$ & $0.198 \pm 0.047$ & $0.230 \pm 0.048$ & $0.208 \pm 0.025$ & $0.178 \pm 0.037$ & $0.3286$ \\
CIFAR-10 & $8$ & $0.378 \pm 0.038$ & $0.373 \pm 0.043$ & $0.353 \pm 0.044$ & $0.278 \pm 0.040$ & $0.5987$ \\
\midrule
Emotion & $2$ & $0.093 \pm 0.033$ & $0.120 \pm 0.022$ & $0.080 \pm 0.023$ & $0.130 \pm 0.035$ & $0.2024$ \\
Emotion & $4$ & $0.088 \pm 0.034$ & $0.095 \pm 0.014$ & $0.090 \pm 0.018$ & $0.105 \pm 0.028$ & $0.3286$ \\
Emotion & $8$ & $0.153 \pm 0.032$ & $0.120 \pm 0.027$ & $0.118 \pm 0.030$ & $0.148 \pm 0.023$ & $0.5987$ \\
\midrule
IMDb & $2$ & $0.178 \pm 0.028$ & $0.173 \pm 0.039$ & $0.158 \pm 0.061$ & $0.150 \pm 0.031$ & $0.2024$ \\
IMDb & $4$ & $0.233 \pm 0.051$ & $0.213 \pm 0.044$ & $0.230 \pm 0.059$ & $0.203 \pm 0.051$ & $0.3286$ \\
IMDb & $8$ & $0.408 \pm 0.050$ & $0.375 \pm 0.046$ & $0.353 \pm 0.046$ & $0.325 \pm 0.037$ & $0.5987$ \\
\midrule
MNIST & $2$ & $0.165 \pm 0.040$ & $0.145 \pm 0.037$ & $0.125 \pm 0.030$ & $0.135 \pm 0.029$ & $0.2024$ \\
MNIST & $4$ & $0.225 \pm 0.047$ & $0.218 \pm 0.045$ & $0.190 \pm 0.060$ & $0.163 \pm 0.025$ & $0.3286$ \\
MNIST & $8$ & $0.340 \pm 0.032$ & $0.303 \pm 0.053$ & $0.285 \pm 0.036$ & $0.265 \pm 0.034$ & $0.5987$ \\
\midrule
TREC-6 & $2$ & $0.095 \pm 0.019$ & $0.123 \pm 0.034$ & $0.108 \pm 0.023$ & $0.085 \pm 0.033$ & $0.2024$ \\
TREC-6 & $4$ & $0.105 \pm 0.031$ & $0.100 \pm 0.032$ & $0.128 \pm 0.039$ & $0.105 \pm 0.036$ & $0.3286$ \\
TREC-6 & $8$ & $0.123 \pm 0.035$ & $0.118 \pm 0.036$ & $0.098 \pm 0.019$ & $0.108 \pm 0.028$ & $0.5987$ \\
\bottomrule
\end{tabular}%
}
\end{table*}

\begin{table}[htb!]
\centering
\caption{\small Direct Gaussian ReRo certificate values used as horizontal reference lines in Figure~\ref{fig:exact-id-radius-grid} and as reference values in the finite-prior audit tables. Here $m=10$, $\delta=10^{-6}$, chance is $1/m=0.1$, and $\Delta/\sigma$ is determined by the analytic Gaussian calibration.}
\label{tab:direct-gaussian-rero-bounds}
\begin{tabular}{cccc}
\toprule
$\varepsilon$ & $\Delta/\sigma$ & $1/m$ & $\gamma_{\mathrm{rel}}$ \\
\midrule
$2$ & $0.4483$ & $0.1000$ & $0.2024$ \\
$4$ & $0.8379$ & $0.1000$ & $0.3286$ \\
$8$ & $1.5315$ & $0.1000$ & $0.5987$ \\
\bottomrule
\end{tabular}
\end{table}

\section{Technical Statements for the Experiments}
\label{sec:technical-statements}

This section records the model-specific facts used by the experimental section. The statements serve two purposes: they specify the Gaussian sensitivity calibration for the ridge-prototype head, and they show how the generic Ball-DP theorem extends to the additional convex heads used in the IMDb comparison. Proofs are deferred to the appendices.

\subsection{Ridge-prototype head}
\label{sec:ridge-statements}

The ridge-prototype head is the closed-form audit case. Its optimizer, sensitivity, and informed-adversary inverse problem are explicit, so it supports exact MAP reconstruction diagnostics without numerical approximation error. The statements below are therefore the ridge-specific instantiation of the general convex-ERM theory, not the full scope of the method.

\begin{theorem}[Exact gradient Lipschitz constant for ridge prototypes]
\label{thm:proto-lz}
Under the label-preserving distance metric defined in Section~\ref{sec:r-adjacency}, Assumption~\ref{ass:lz} holds for $\ell(\mu;(z,y))=\|z-\mu_y\|_2^2$ 
with the exact constant $L_z=2$.
\end{theorem}

\begin{corollary}[Exact $r$-adjacency sensitivity]
\label{cor:proto-sensitivity}
Let $\widehat\mu(D)$ be the unique minimizer of~\eqref{eq:ridge-proto-objective}. The count-worst-case sensitivity at radius $r$ is
\[
\Delta_{\mathrm{proto}}(r)=
\sup_{D\sim_rD'}\|\widehat\mu(D)-\widehat\mu(D')\|_2
=
\frac{2r}{2+\lambda n}.
\]
This is sharper than the generic upper bound from Theorem~\ref{thm:erm-sensitivity}.
\end{corollary}

Corollary~\ref{cor:proto-sensitivity} is count-worst-case: it allows the changed record to be the only record in its class. If class counts are public, or are fixed component metadata of the label-preserving adjacency graph, the calibration can instead use the smallest nonzero public class count.

\begin{proposition}[Count-aware ridge-prototype sensitivity]
\label{prop:proto-count-aware-sensitivity}
Let $n_c$ be the public number of records in class $c$, and let
\[
n_{\min}:=\min\{n_c:n_c>0\}.
\]
Conditional on these counts, the exact $r$-adjacency sensitivity of the ridge-prototype optimizer is
\[
\Delta_{\mathrm{proto}}(r;n_1,\ldots,n_C)
=
\max_{c:n_c>0}\frac{2r}{2n_c+\lambda n}
=
\frac{2r}{2n_{\min}+\lambda n}.
\]
\end{proposition}

The maximum is attained by a class with count $n_{\min}$, since one changed record moves rare-class prototypes the most. Because $n_{\min}\ge1$, the count-aware calibration is never looser than the count-worst-case bound $2r/(2+\lambda n)$, and it is strictly tighter whenever $n_{\min}>1$.

\begin{theorem}[Exact reconstruction from a noiseless release]
\label{thm:proto-exact-recon}
Let $D=D^-\cup\{(z,y)\}$ and suppose an informed adversary knows $D^-$, $n$, $\lambda$, and the noiseless optimum $\widehat\mu(D)$. If the missing label $y$ is known, then
\[
z=
\frac{2(n_y^-+1)+\lambda n}{2}\,\widehat\mu_y(D)
-
\sum_{(z',y')\in D^-:\,y'=y}z',
\]
where $n_y^-$ is the number of public label-$y$ examples.

If the missing label is unknown, define
\[
\mu_c^0(D^-)
:=
\frac{2\sum_{(z',y')\in D^-:\,y'=c}z'}{2n_c^-+\lambda n}
\]
for every class $c$ with $n_c^->0$. If there is a unique class $c$ such that $\widehat\mu_c(D)\ne \mu_c^0(D^-)$, then $y=c$ and the same formula reconstructs $z$.
\end{theorem}

The only label-identification obstruction is the degenerate event in which the missing example equals the $D^-$-implied prototype of its own class. In that case, adding the missing record does not move the corresponding released prototype.

\begin{corollary}[Known-label noisy inversion]
\label{cor:ridge-noisy-inversion}
Fix $D^-$ and suppose the missing label $y$ is known. Let
\[
A_y:=2(n_y^-+1)+\lambda n,
\qquad
S_y^-:=\sum_{(z',y')\in D^-:\,y'=y}z'.
\]
For the Gaussian release, define
\[
W_y=(A_y/2)\widetilde\mu_y-S_y^-.
\]
If the hidden record is $Z$, then
\[
W_y=Z+\eta_y,
\qquad
\eta_y\sim\mathcal N(0,\tau_y^2 I),
\qquad
\tau_y=(A_y/2)\sigma.
\]
\end{corollary}

Thus, for a uniform finite prior with common known label $y$, the exact MAP attack is nearest-neighbor decoding of $W_y$ over the candidate support.

\subsection{Logistic and SVM heads}
\label{sec:secondary-head-statements}

The same ERM sensitivity theorem applies to the additional convex heads used in the experiments once Assumption~\ref{ass:lz} is verified for each model. For the main text, we record only the resulting constants needed for the Ball-DP calibration; the proof of Proposition~\ref{prop:secondary-lz-summary} is deferred to Appendix~\ref{app:proof-secondary-lz}.

Throughout this subsection we use the label-preserving metric, assume $\|z\|_2\le B$, and write
\[
\widetilde B:=\sqrt{B^2+1}
\]
for the augmented feature norm.

\begin{proposition}[Explicit $L_z$ constants for common convex heads]
\label{prop:secondary-lz-summary}
For same-label substitutions, Assumption~\ref{ass:lz} holds for the following convex heads with the corresponding dataset-independent constants:
\begin{itemize}[leftmargin=1.4em]
    \item \emph{Squared hinge loss:}
    \[
    L_z \le 2+4B\sqrt{2/\lambda}.
    \]

    \item \emph{Binary logistic loss:}
    \[
    L_z \le 1+\frac{\widetilde B^2}{4\lambda}.
    \]

    \item \emph{Softmax logistic loss:}
    \[
    L_z \le \sqrt 2\left(1+\frac{\widetilde B^2}{2\lambda}\right).
    \]
\end{itemize}

Consequently, for any of these heads, Theorem~\ref{thm:erm-sensitivity} gives
\[
\|\widehat\theta(D)-\widehat\theta(D')\|
\le
\frac{L_z}{\lambda n}r
\qquad
\text{for all } D\sim_r D',
\]
with $L_z$ chosen from the corresponding case above.
\end{proposition}
\section{Related Work}
\label{sec:related}

\paragraph{Adjacency design beyond the global worst case.}
The choice of adjacency is central to the meaning of DP \cite{dwork2014book}. Metric formulations such as $d_X$-privacy and later work on metric DP replace the uniform adjacency graph by guarantees that depend on a record metric, especially in location and embedding domains \cite{chatzikokolakis2013metrics,imola2022metricdp}. Ball-DP is simpler and more policy-oriented: it keeps the standard $(\varepsilon,\delta)$ semantics unchanged and restricts the protected substitutions to a declared radius-$r$ neighborhood.

\paragraph{Private convex ERM and output perturbation.}
Private ERM under convexity has a long history, including output- and objective-perturbation methods \cite{chaudhuri2011erm}. In the Gaussian output-perturbation regime, analytic calibration improves over the classical tail-bound formula \cite{balle2018analytic}. Our generic ERM sensitivity theorem is the local counterpart of the usual output-perturbation analysis. Ridge prototypes give an exact equality rather than only an upper bound, while the binary-logistic, softmax-logistic, and squared-hinge SVM constants show how the same local-calibration argument extends to common convex linear heads.

\paragraph{Reconstruction attacks and reconstruction robustness.}
Balle, Cherubin, and Hayes formalized the informed-adversary reconstruction setting and showed that convex models can admit exact or near-exact inversion from noiseless releases \cite{balle2022informed}. Hayes, Mahloujifar, and Balle derived reconstruction bounds and matching attacks for DP-SGD \cite{hayes2023bounding}. The hypothesis-testing view of privacy, developed in the $f$-DP line of work, provides a natural route from privacy guarantees to attack bounds \cite{dong2022fdp}. Kaissis \emph{et al.} used this perspective to derive direct reconstruction bounds for Gaussian and subsampled mechanisms \cite{kaissis2023rero}. Our Gaussian Ball-ReRo result is the local, radius-thresholded analogue of that direct Gaussian testing argument.

\paragraph{Prototype-based private learning.}
Prototype representations are common in classical nearest-centroid classifiers and modern embedding-based learning \cite{snell2017prototypical}. Recent private learning work has also explored prototype releases in transfer-learning settings \cite{wahdany2025dppl}. Our use of prototypes is different: the model is chosen as the closed-form audit case because its convex structure yields exact local sensitivity and an exact informed-adversary audit. The broader convex-head contribution is captured by the generic ERM theorem and the explicit logistic/SVM constants.

\section{Summary and Future Directions}
\label{sec:conclusion}

This paper introduced Ball-DP, a radius-restricted adjacency relation that makes the locality of a privacy guarantee an explicit, auditable policy choice rather than an implicit byproduct of noise tuning. The standard $(\varepsilon, \delta)$-indistinguishability semantics are preserved within the declared neighborhood; only the scope of the guarantee changes. For strongly convex ERM, we showed that sensitivity under $r$-adjacency scales linearly with $r$, giving a direct route to recalibrated Gaussian output perturbation. The theorem is instantiated for ridge prototypes, binary logistic regression, softmax logistic regression, and squared-hinge SVMs through explicit record-Lipschitz constants. For ridge-regularized prototype classifiers, the sensitivity, the noiseless informed-adversary inversion, and the finite-prior Gaussian MAP attack are all exact, allowing privacy calibration and reconstruction auditing to be studied without confounding approximation error.

The experiments show that this local calibration recovers meaningful accuracy, while empirical exact-identification rates remain below the theorem-backed Ball-ReRo certificates. The additional IMDb convex-head sweep shows the same utility--certificate pattern for binary logistic and squared-hinge SVM heads, supporting the claim that the method is a convex-ERM calibration principle rather than a prototype-only construction. The current analysis is, however, confined to convex, one-shot Gaussian releases with public calibration metadata; it does not address label-changing substitutions, transcript-level releases from private SGD, or nonconvex representation-learning pipelines. Extending the Ball-DP accounting to these settings, particularly multi-step DP-SGD transcripts where sensitivity no longer admits a closed form, is the main direction for future work.

\bibliographystyle{abbrvnat}
\bibliography{ref}

\section*{Ethics considerations}
This paper studies privacy guarantees and reconstruction attacks for machine-learning models trained on potentially sensitive data. The experiments use fixed embedding representations and theorem-aligned attack protocols whose purpose is to audit privacy leakage under controlled assumptions, not to facilitate unauthorized extraction of personal data from deployed systems. Empirical attack evaluations should be run only on datasets for which the researchers have authorization, and released artifacts should omit raw private records. No human-subject intervention is involved in the analysis reported here. The experiments evaluate candidate identification in embedding space rather than recovered raw images or text.

\appendix

\section{Proofs for Main Results}
\label{sec:core-learning-proofs}


\subsection{Proof of Proposition~\ref{prop:bounded-replacement-special-case}}
\begin{proof}
If $\|z\|_2\le B$ and $\|z'\|_2\le B$, the triangle inequality gives
\[
\|z-z'\|_2\le \|z\|_2+\|z'\|_2\le 2B.
\]
Thus every admissible single-record replacement is $2B$-adjacent, so the bounded-replacement adjacency graph is contained in the Ball adjacency graph at radius $2B$.
\end{proof}

\subsection{Proof of Theorem~\ref{thm:erm-sensitivity}}
\label{app:proof-erm-sensitivity}

We first state a generic optimizer-perturbation bound that does not depend on the metric or on any Lipschitz-in-sample assumption.

\begin{lemma}[Single-record optimizer perturbation]
\label{lem:single-record-erm-perturbation}
Let $D,D'\in\mathcal Z^n$ differ only at index $j$, with records $z_j$ and $z'_j$. Suppose $F_D$ is differentiable and $\lambda$-strongly convex, and suppose the minimizers $\widehat\theta(D)$ and $\widehat\theta(D')$ lie in a parameter set $\Theta$. Then
\[
\|\widehat\theta(D)-\widehat\theta(D')\|_2
\le
\frac{1}{\lambda n}
\sup_{\theta\in\Theta}
\left\|
\nabla_\theta\ell(\theta;z_j)
-
\nabla_\theta\ell(\theta;z'_j)
\right\|_2.
\]
\end{lemma}

\begin{proof}[Proof of Lemma~\ref{lem:single-record-erm-perturbation}]
Let
\[
u:=\widehat\theta(D),
\qquad
v:=\widehat\theta(D').
\]
If $u=v$, the result is immediate, so assume $u\ne v$. Since $F_D$ is differentiable and $\lambda$-strongly convex, its gradient is $\lambda$-strongly monotone:
\[
\big\langle \nabla F_D(a)-\nabla F_D(b),\,a-b\big\rangle
\ge
\lambda\|a-b\|_2^2
\qquad
\text{for all }a,b.
\]
Apply this with $a=u$ and $b=v$. Since $u$ minimizes $F_D$, $\nabla F_D(u)=0$, and therefore
\[
\lambda\|u-v\|_2^2
\le
\big\langle -\nabla F_D(v),\,u-v\big\rangle
\le
\|\nabla F_D(v)\|_2\,\|u-v\|_2.
\]
Dividing by $\|u-v\|_2$ gives
\[
\lambda\|u-v\|_2
\le
\|\nabla F_D(v)\|_2.
\]
Since $v$ minimizes $F_{D'}$, $\nabla F_{D'}(v)=0$, so
\[
\|\nabla F_D(v)\|_2
=
\|\nabla F_D(v)-\nabla F_{D'}(v)\|_2.
\]
The datasets differ only at index $j$, hence all unchanged records and the regularizer cancel in the gradient difference:
\[
\nabla F_D(v)-\nabla F_{D'}(v)
=
\frac{1}{n}
\left(
\nabla_\theta\ell(v;z_j)
-
\nabla_\theta\ell(v;z'_j)
\right).
\]
Because $v\in\Theta$, we have
\[
\left\|
\nabla_\theta\ell(v;z_j)
-
\nabla_\theta\ell(v;z'_j)
\right\|_2
\le
\sup_{\theta\in\Theta}
\left\|
\nabla_\theta\ell(\theta;z_j)
-
\nabla_\theta\ell(\theta;z'_j)
\right\|_2.
\]
Combining the previous displays yields
\[
\|\widehat\theta(D)-\widehat\theta(D')\|_2
\le
\frac{1}{\lambda n}
\sup_{\theta\in\Theta}
\left\|
\nabla_\theta\ell(\theta;z_j)
-
\nabla_\theta\ell(\theta;z'_j)
\right\|_2,
\]
as claimed.
\end{proof}

\begin{proof}[Proof of Theorem~\ref{thm:erm-sensitivity}]
Let $D\sim_rD'$ and let $j$ be the differing coordinate. By Lemma~\ref{lem:single-record-erm-perturbation},
\[
\|\widehat\theta(D)-\widehat\theta(D')\|_2
\le
\frac{1}{\lambda n}
\sup_{\theta\in\Theta}
\left\|
\nabla_\theta\ell(\theta;z_j)
-
\nabla_\theta\ell(\theta;z'_j)
\right\|_2.
\]
Assumption~\ref{ass:lz} gives
\[
\sup_{\theta\in\Theta}
\left\|
\nabla_\theta\ell(\theta;z_j)
-
\nabla_\theta\ell(\theta;z'_j)
\right\|_2
\le
L_z\,d(z_j,z'_j).
\]
Since $D\sim_rD'$, we have $d(z_j,z'_j)\le r$. Therefore
\[
\|\widehat\theta(D)-\widehat\theta(D')\|_2
\le
\frac{L_z}{\lambda n}r.
\]
\end{proof}

\subsection{Proof of Proposition~\ref{prop:outside-ball}}
\label{app:proof-outside-ball}

\begin{proof}
Fix $t>0$ and consider any pair $D\sim_tD'$. By Theorem~\ref{thm:erm-sensitivity} with radius $t$,
\[
\|\widehat\theta(D)-\widehat\theta(D')\|_2
\le
\Delta_t
:=
\frac{L_z}{\lambda n}t.
\]
Thus the two releases are Gaussian distributions with common covariance $\sigma^2I_p$ and means separated by at most $\Delta_t$. The Gaussian privacy profile therefore implies that the release is $(\varepsilon_t,\delta_t)$-DP on the $t$-adjacency relation whenever
\[
\delta_t
\ge
\Phi\left(\frac{\Delta_t}{2 \sigma}-\frac{\varepsilon_t \sigma}{\Delta_t}\right)-e^{\varepsilon} \Phi\left(-\frac{\Delta_t}{2 \sigma}-\frac{\varepsilon \sigma}{\Delta_t}\right)
\]

For the closed-form statement, assume
\[
\sigma\ge \frac{\Delta_r}{\varepsilon}\sqrt{2\log(1.25/\delta)},
\qquad
\Delta_r:=\frac{L_z}{\lambda n}r.
\]
Set $\varepsilon_t=\varepsilon t/r$. Since
\[
\frac{\Delta_t}{\varepsilon_t}
=
\frac{(L_z/(\lambda n))t}{\varepsilon t/r}
=
\frac{L_zr}{\lambda n\,\varepsilon}
=
\frac{\Delta_r}{\varepsilon},
\]
the chosen $\sigma$ also satisfies
\[
\sigma\ge
\frac{\Delta_t}{\varepsilon_t}
\sqrt{2\log(1.25/\delta)}.
\]
Whenever $\varepsilon_t$ lies in the usual parameter range for the classical Gaussian mechanism bound, that bound gives $(\varepsilon_t,\delta)$-DP on the $t$-adjacency relation.
\end{proof}

\subsection{Proofs for the Ridge-Prototype Model}
\label{app:proof-ridge}

\subsubsection{Proof of Theorem~\ref{thm:proto-lz}}

\begin{proof}
We identify the parameter $\mu=(\mu_1,\dots,\mu_C)$ with the concatenated vector in $\mathbb{R}^{Cd}$ obtained by stacking the class blocks. Under this identification, the per-example loss is
\[
\ell(\mu;(z,y))=\|z-\mu_y\|_2^2.
\]
Fix a class index $c\in\{1,\dots,C\}$. The gradient of $\ell$ with respect to block $\mu_c$ is
\[
\nabla_{\mu_c}\ell(\mu;(z,y))
=
2(\mu_y-z)\,\ind{c=y}.
\]
Now take two records $(z,y)$ and $(z',y')$ and subtract their blockwise gradients:
\[
\begin{aligned}
&\nabla_{\mu_c}\ell(\mu;(z,y))
-
\nabla_{\mu_c}\ell(\mu;(z',y')) \\
&\qquad =
2(\mu_y-z)\,\ind{c=y}
-
2(\mu_{y'}-z')\,\ind{c=y'}.
\end{aligned}
\]
Under the label-preserving metric of Equation~\eqref{eq:label-preserving-metric}, finite record distance occurs only when $y=y'$. Therefore the only nonvacuous case is $y=y'$. In that case the display above simplifies to
\[
\nabla_{\mu_c}\ell(\mu;(z,y))-
\nabla_{\mu_c}\ell(\mu;(z',y))
=
2(z'-z)\,\ind{c=y}.
\]
This means that all blocks are zero except the block indexed by $y$. The full gradient difference therefore has exactly one nonzero block, equal to $2(z'-z)$. Its Euclidean norm is thus
\[
\|\nabla_\mu \ell(\mu;(z,y))-
\nabla_\mu \ell(\mu;(z',y))\|_2
=
2\|z-z'\|_2.
\]
By the definition of the label-preserving metric,
\[
d\big((z,y),(z',y)\big)=\|z-z'\|_2.
\]
Substituting this identity into the previous display yields
\[
\|\nabla_\mu \ell(\mu;(z,y))-
\nabla_\mu \ell(\mu;(z',y))\|_2
=
2\,d\big((z,y),(z',y)\big).
\]
Hence Assumption~\ref{ass:lz} holds with $L_z=2$.

To see that the constant is exact, choose any same-label pair with $z\neq z'$. The last display becomes an equality, so no smaller constant can work uniformly.
\end{proof}

\subsubsection{Proof of Corollary~\ref{cor:proto-sensitivity}}

\begin{proof}
Let
\[
\widehat\mu(D)\in\arg\min_\mu F_D(\mu)
\]
be the unique minimizer of the ridge-prototype objective. For each class $c$, define the class sum and class count
\[
S_c(D):=\sum_{i:y_i=c} z_i,
\qquad
n_c(D):=\#\{i:y_i=c\}.
\]
Because the objective separates over the class blocks, we can write
\[
F_D(\mu)
=
\sum_{c=1}^C
\left[
\frac{1}{n}\sum_{i:y_i=c}\|z_i-\mu_c\|_2^2
+
\frac{\lambda}{2}\|\mu_c\|_2^2
\right].
\]
Fix a class $c$. Differentiating the $c$th summand with respect to $\mu_c$ gives
\[
\nabla_{\mu_c}F_D(\mu)
=
\frac{2}{n}\sum_{i:y_i=c}(\mu_c-z_i)+\lambda\mu_c.
\]
Expand the sum:
\[
\nabla_{\mu_c}F_D(\mu)
=
\frac{2n_c(D)}{n}\mu_c-
\frac{2}{n}S_c(D)+\lambda\mu_c.
\]
Group the $\mu_c$ terms:
\[
\nabla_{\mu_c}F_D(\mu)
=
\left(\frac{2n_c(D)}{n}+\lambda\right)\mu_c-
\frac{2}{n}S_c(D).
\]
At the optimum this gradient is zero, so
\[
\left(\frac{2n_c(D)}{n}+\lambda\right)\widehat\mu_c(D)=\frac{2}{n}S_c(D).
\]
Multiply both sides by $n$:
\[
\bigl(2n_c(D)+\lambda n\bigr)\widehat\mu_c(D)=2S_c(D).
\]
Therefore the closed-form optimizer is
\begin{equation}
\label{eq:proto-closed-form}
\widehat\mu_c(D)=\frac{2S_c(D)}{2n_c(D)+\lambda n}.
\end{equation}

Now let $D'$ be obtained from $D$ by replacing one record $(z,y)$ by $(z',y)$ with $\|z-z'\|_2\le r$. Because the label is preserved, the class counts do not change:
\[
n_c(D')=n_c(D)\qquad \text{for every class }c.
\]
The class sums are also unchanged for all classes except the affected class $y$:
\[
S_c(D')=S_c(D)\qquad \text{for }c\neq y,
\]
while for the affected class,
\[
S_y(D')=S_y(D)-z+z'.
\]
Substituting these identities into~\eqref{eq:proto-closed-form} shows that
\[
\widehat\mu_c(D')=\widehat\mu_c(D)\qquad \text{for every }c\neq y,
\]
and
\[
\widehat\mu_y(D')-
\widehat\mu_y(D)
=
\frac{2\bigl(S_y(D')-S_y(D)\bigr)}{2n_y(D)+\lambda n}
=
\frac{2(z'-z)}{2n_y(D)+\lambda n}.
\]
Thus the full parameter difference has exactly one nonzero block, namely the $y$th block. Its norm is therefore
\[
\|\widehat\mu(D')-\widehat\mu(D)\|_2
=
\|\widehat\mu_y(D')-\widehat\mu_y(D)\|_2
=
\frac{2\|z'-z\|_2}{2n_y(D)+\lambda n}.
\]
Since the changed record itself belongs to class $y$, we have $n_y(D)\ge 1$. Hence
\[
\|\widehat\mu(D')-\widehat\mu(D)\|_2
\le
\frac{2r}{2n_y(D)+\lambda n}
\le
\frac{2r}{2+\lambda n}.
\]
This proves the upper bound
\[
\Delta_{\mathrm{proto}}(r)\le \frac{2r}{2+\lambda n}.
\]

To prove equality, choose a dataset in which the modified record is the \emph{only} example from its class, so that $n_y(D)=1$. Choose the replacement so that $\|z-z'\|_2=r$. For this pair,
\[
\|\widehat\mu(D')-\widehat\mu(D)\|_2
=
\frac{2r}{2+\lambda n}.
\]
Therefore the upper bound is attained, and
\[
\Delta_{\mathrm{proto}}(r)=\frac{2r}{2+\lambda n}.
\]
\end{proof}

\subsubsection{Proof of Proposition~\ref{prop:proto-count-aware-sensitivity}}

\begin{proof}
Under the label-preserving metric, adjacent datasets differ by replacing $(z,y)$ with $(z',y)$, so the class labels and hence the count vector are unchanged along every adjacency edge. The proof of Corollary~\ref{cor:proto-sensitivity} gives, for a replacement in class $c$,
\[
\|\widehat\mu(D')-\widehat\mu(D)\|_2
=
\frac{2\|z'-z\|_2}{2n_c+\lambda n}.
\]
Conditional on the public count vector, the only remaining choices are the class $c$ and the replacement distance. The radius policy imposes $\|z'-z\|_2\le r$, so the largest possible movement is
\[
\max_{c:n_c>0}\frac{2r}{2n_c+\lambda n}.
\]
The denominator is smallest for a class with the smallest nonzero count $n_{\min}$, which gives $2r/(2n_{\min}+\lambda n)$. Equality is attained by choosing such a class and a replacement at distance exactly $r$.
\end{proof}

\subsubsection{Proof of Theorem~\ref{thm:proto-exact-recon}}

\begin{proof}
For each class $c\in\{1,\dots,C\}$, define the class sum and class count in the known dataset $D^-$ by
\[
\begin{aligned}
S_c^-
&:= \sum_{\substack{(z',y')\in D^-\\ y'=c}} z',
\\
n_c^-
&:= \#\{(z',y')\in D^- : y'=c\}.
\end{aligned}
\]
Since the full dataset is $D=D^-\cup\{(z,y)\}$, the full-data class statistics are
\[
S_c(D)=S_c^-+z\,\ind{c=y},
\qquad
n_c(D)=n_c^-+\ind{c=y}.
\]

As shown in the proof of Corollary~\ref{cor:proto-sensitivity}, the released prototype satisfies
\[
\widehat\mu_c(D)=\frac{2S_c(D)}{2n_c(D)+\lambda n}.
\]
We now prove the two claims separately.

\paragraph{Part 1: exact record reconstruction when the label is known.}
Assume the adversary knows the missing label $y$. Then for the affected class,
\[
S_y(D)=S_y^-+z,
\qquad
n_y(D)=n_y^-+1.
\]
Substituting these identities into the closed-form optimizer gives
\[
\widehat\mu_y(D)=\frac{2(S_y^-+z)}{2(n_y^-+1)+\lambda n}.
\]
Multiply both sides by the denominator:
\[
\bigl(2(n_y^-+1)+\lambda n\bigr)\widehat\mu_y(D)=2(S_y^-+z).
\]
Expand the right-hand side:
\[
\bigl(2(n_y^-+1)+\lambda n\bigr)\widehat\mu_y(D)=2S_y^-+2z.
\]
Subtract $2S_y^-$ from both sides:
\[
\bigl(2(n_y^-+1)+\lambda n\bigr)\widehat\mu_y(D)-2S_y^-=2z.
\]
Finally divide by $2$:
\[
z=\frac{2(n_y^-+1)+\lambda n}{2}\,\widehat\mu_y(D)-S_y^-.
\]
Since $S_y^-=
\sum_{(z',y')\in D^-:\,y'=y} z'$, this is exactly the displayed reconstruction formula.

\paragraph{Part 2: label identification from the changed prototype.}
For any class $c\neq y$, the missing record does not affect class $c$. Thus
\[
S_c(D)=S_c^-,
\qquad
n_c(D)=n_c^-.
\]
Substituting into the closed-form optimizer gives
\[
\widehat\mu_c(D)=\frac{2S_c^-}{2n_c^-+\lambda n}.
\]
Define the $D^-$-implied baseline prototype
\[
\mu_c^-:=\frac{2S_c^-}{2n_c^-+\lambda n}.
\]
Then we have shown that
\[
\widehat\mu_c(D)=\mu_c^- \qquad \text{for every }c\neq y.
\]
Therefore, if there is a unique class $c^\star$ such that
\[
\widehat\mu_{c^\star}(D)\neq \mu_{c^\star}^-,
\]
that class must be the true missing class: $c^\star=y$. Once the label is identified, Part~1 reconstructs the record exactly.

It remains to characterize the only case where the missing class fails to move. Suppose
\[
\widehat\mu_y(D)=\mu_y^-.
\]
Using the formulas above, this means
\[
\frac{2(S_y^-+z)}{2(n_y^-+1)+\lambda n}
=
\frac{2S_y^-}{2n_y^-+\lambda n}.
\]
Cancel the common factor $2$:
\[
\frac{S_y^-+z}{2(n_y^-+1)+\lambda n}
=
\frac{S_y^-}{2n_y^-+\lambda n}.
\]
Cross-multiply:
\[
(S_y^-+z)(2n_y^-+\lambda n)
=
S_y^-(2n_y^-+2+\lambda n).
\]
Expand the left-hand side:
\[
S_y^-(2n_y^-+\lambda n)+z(2n_y^-+\lambda n)
=
S_y^-(2n_y^-+\lambda n)+2S_y^-.
\]
Subtract the common term $S_y^-(2n_y^-+\lambda n)$ from both sides:
\[
z(2n_y^-+\lambda n)=2S_y^-.
\]
Finally divide by $2n_y^-+\lambda n$ to get
\[
z=\frac{2S_y^-}{2n_y^-+\lambda n}=\mu_y^-.
\]
Thus,
\[
\widehat\mu_y(D)=\mu_y^- \quad \Longrightarrow \quad z=\mu_y^-.
\]
The converse is immediate: if $z=\mu_y^-$, then substituting this into the formula for $\widehat\mu_y(D)$ gives back $\mu_y^-$. Hence
\[
\widehat\mu_y(D)=\mu_y^- \quad \Longleftrightarrow \quad z=\mu_y^-.
\]
So the only obstruction to label identification is the degenerate case in which the missing record equals the $D^-$-implied prototype of its class.
\end{proof}

\subsubsection{Proof of Corollary~\ref{cor:ridge-noisy-inversion}}

\begin{proof}
By the closed-form optimizer in~\eqref{eq:proto-closed-form}, if the missing record is $(Z,y)$ then
\[
\widehat\mu_y(D^-\cup\{(Z,y)\})=\frac{2(S_y^-+Z)}{A_y},
\qquad
A_y=2(n_y^-+1)+\lambda n.
\]
The Gaussian release gives
\[
\widetilde\mu_y=\frac{2(S_y^-+Z)}{A_y}+\xi_y,
\qquad
\xi_y\sim\mathcal N(0,\sigma^2 I).
\]
Multiplying by $A_y/2$ and subtracting $S_y^-$ yields
\[
W_y=\frac{A_y}{2}\widetilde\mu_y-S_y^-
=Z+\frac{A_y}{2}\xi_y.
\]
Thus $W_y=Z+\eta_y$ with $\eta_y\sim\mathcal N(0,(A_y\sigma/2)^2I)$.
For a uniform finite prior with common known label, the Gaussian likelihood is proportional to
\[
\exp\!\left(-\frac{\|W_y-z_i\|_2^2}{2\tau_y^2}\right),
\]
so MAP decoding is nearest-neighbor decoding over the candidate set.
\end{proof}

\section{Proofs for Ball-ReRo}
\label{app:proof-ball-rero}

\subsection{Proof of Theorem~\ref{thm:ball-dp-rero}}
\label{app:proof-ball-dp-rero}

\begin{proof}
For each record $z\in\mathcal{Z}$, define the completed dataset
\[
D_z:=D^-\cup\{z\},
\]
and let
\[
P_z:=\Law\big(M(D_z)\big)
\]
be the law of the mechanism output under $D_z$. Define the success indicator
\[
A(z,\theta):=\ind{\rho(z,R(\theta,D^-))\le \eta}.
\]
The reconstruction success probability is
\[
p_{\mathrm{succ}}
=
\mathbb{E}_{Z\sim\pi_{u,r}}
\mathbb{E}_{\theta\sim P_Z}[A(Z,\theta)].
\]
Because $\pi_{u,r}$ is supported on $\ball{d}{u,r}$, every $z$ in its support satisfies $d(z,u)\le r$. Hence $D_z$ and $D_u$ are r-adjacent. Since $M$ is $(\varepsilon,\delta)$ r-adjacency DP, for every such $z$ and every measurable event $E$,
\[
P_z(E)\le e^{\varepsilon}P_u(E)+\delta.
\]
Apply this inequality to
\[
E_z:=\{\theta: \rho(z,R(\theta,D^-))\le \eta\}.
\]
Then
\[
\mathbb{E}_{\theta\sim P_z}[A(z,\theta)]
= P_z(E_z)
\le e^{\varepsilon}P_u(E_z)+\delta.
\]
Averaging over $Z\sim\pi_{u,r}$ gives
\[
p_{\mathrm{succ}}
\le
 e^{\varepsilon}
\mathbb{E}_{Z\sim\pi_{u,r}}
\mathbb{E}_{\theta\sim P_u}[A(Z,\theta)]
+\delta.
\]
By Tonelli's theorem~\citep[Chapter~2]{folland1999real},
\[
p_{\mathrm{succ}}
\le
 e^{\varepsilon}
\mathbb{E}_{\theta\sim P_u}
\Pr_{Z\sim\pi_{u,r}}[\rho(Z,R(\theta,D^-))\le\eta]
+\delta.
\]
For each fixed $\theta$, the output constraint gives $R(\theta,D^-)\in\mathcal A_{u,r}$. Therefore
\[
\begin{aligned}
&\Pr_{Z\sim\pi_{u,r}}
  [\rho(Z,R(\theta,D^-))\le\eta] \\
&\quad\le
\sup_{a\in\mathcal A_{u,r}}
\Pr_{Z\sim\pi_{u,r}}[\rho(Z,a)\le\eta]
=
\kappa.
\end{aligned}
\]
Substituting this bound yields
\[
p_{\mathrm{succ}}
\le e^{\varepsilon}\kappa+\delta,
\]
which is the claimed Ball-ReRo guarantee.
\end{proof}

\subsection{Proof of Lemma~\ref{lem:gaussian-ball-rdp}}
\label{app:proof-gaussian-ball-rdp}

\begin{proof}
Fix a pair of r-adjacent datasets $D\sim_r D'$. Write
\[
\mu_0:=f(D),
\qquad
\mu_1:=f(D').
\]
Then the two released distributions are
\[
M(D)\sim \mathcal{N}(\mu_0,\sigma^2 I),
\qquad
M(D')\sim \mathcal{N}(\mu_1,\sigma^2 I).
\]
For two Gaussians with the same covariance matrix, the order-$\alpha$ R\'enyi divergence is
\[
D_\alpha\big(\mathcal{N}(\mu_0,\sigma^2 I)\,\|\,\mathcal{N}(\mu_1,\sigma^2 I)\big)
=
\frac{\alpha}{2\sigma^2}\|\mu_0-\mu_1\|_2^2.
\]
Applying this identity gives
\[
D_\alpha\big(M(D)\,\|\,M(D')\big)
=
\frac{\alpha}{2\sigma^2}\|f(D)-f(D')\|_2^2.
\]
By the definition of r-adjacency sensitivity,
\[
\|f(D)-f(D')\|_2\le \Delta_2(r).
\]
Substituting this into the previous display proves
\[
D_\alpha\big(M(D)\,\|\,M(D')\big)
\le
\frac{\alpha\,\Delta_2(r)^2}{2\sigma^2}.
\]
\end{proof}

\subsection{Proof of Theorem~\ref{thm:ball-rdp-rero}}
\label{app:proof-ball-rdp-rero}

\begin{proof}
For each candidate record $z$, let
\[
D_z:=D^-\cup\{z\},
\qquad
P_z:=\Law\big(M(D_z)\big).
\]
Define
\[
A(z,\theta):=\ind{\rho(z,R(\theta,D^-))\le \eta}.
\]
The reconstruction success probability is
\[
p_{\mathrm{succ}}
=
\int_{\mathcal{Z}}\int_{\Theta} A(z,\theta)\,P_z(d\theta)\,\pi_{u,r}(dz).
\]
Because $\pi_{u,r}$ is supported on $\ball{d}{u,r}$, every $z$ in its support satisfies $D_z\sim_r D_u$. By Ball-RDP,
\[
D_\alpha(P_z\|P_u)\le \varepsilon_\alpha.
\]
In particular, $P_z$ is absolutely continuous with respect to $P_u$ for all relevant $z$. Let
\[
\Lambda_z(\theta):=\frac{dP_z}{dP_u}(\theta).
\]
Rewrite $p_{\mathrm{succ}}$ using $P_u$ as the reference measure:
\[
p_{\mathrm{succ}}
=
\int_{\Theta}
\left[
\int_{\mathcal{Z}} A(z,\theta)\Lambda_z(\theta)\,\pi_{u,r}(dz)
\right]
P_u(d\theta).
\]
Fix $\theta$ and apply H\"older's inequality with conjugate exponents $\alpha$ and $\alpha'=\alpha/(\alpha-1)$:
\[
\begin{aligned}
&\int_{\mathcal{Z}}
A(z,\theta)\Lambda_z(\theta)\,\pi_{u,r}(dz) \\
&\quad\le
\left(
\int_{\mathcal{Z}} A(z,\theta)\,\pi_{u,r}(dz)
\right)^{1/\alpha'}
\left(
\int_{\mathcal{Z}} \Lambda_z(\theta)^\alpha\,
\pi_{u,r}(dz)
\right)^{1/\alpha}.
\end{aligned}
\]
For fixed $\theta$, the output constraint gives $R(\theta,D^-)\in\mathcal A_{u,r}$, so the first factor is bounded by
\[
\int_{\mathcal{Z}} A(z,\theta)\,\pi_{u,r}(dz)
=
\Pr_{Z\sim\pi_{u,r}}[\rho(Z,R(\theta,D^-))\le \eta]
\le \kappa.
\]
Therefore
\[
p_{\mathrm{succ}}
\le
\kappa^{1/\alpha'}
\int_{\Theta}
\left(
\int_{\mathcal{Z}} \Lambda_z(\theta)^\alpha\,\pi_{u,r}(dz)
\right)^{1/\alpha}
P_u(d\theta).
\]
Let
\[
H(\theta):=\int_{\mathcal{Z}} \Lambda_z(\theta)^\alpha\,\pi_{u,r}(dz).
\]
Since $x\mapsto x^{1/\alpha}$ is concave, Jensen's inequality gives
\[
\left(\int_{\Theta} H(\theta)^{1/\alpha}P_u(d\theta)\right)^\alpha
\le
\int_{\Theta}H(\theta)P_u(d\theta).
\]
Consequently,
\[
\left(\frac{p_{\mathrm{succ}}}{\kappa^{1/\alpha'}}\right)^\alpha
\le
\int_{\mathcal{Z}}
\left[
\int_{\Theta} \Lambda_z(\theta)^\alpha P_u(d\theta)
\right]
\pi_{u,r}(dz).
\]
By the definition of R\'enyi divergence,
\[
\int_{\Theta} \Lambda_z(\theta)^\alpha P_u(d\theta)
=
\exp\big((\alpha-1)D_\alpha(P_z\|P_u)\big)
\le e^{(\alpha-1)\varepsilon_\alpha}.
\]
Thus
\[
\left(\frac{p_{\mathrm{succ}}}{\kappa^{1/\alpha'}}\right)^\alpha
\le e^{(\alpha-1)\varepsilon_\alpha}.
\]
Taking the $\alpha$th root and using $1/\alpha'=(\alpha-1)/\alpha$ yields
\[
p_{\mathrm{succ}}
\le
\big(\kappa e^{\varepsilon_\alpha}\big)^{(\alpha-1)/\alpha}.
\]
Since a probability is at most one, the claimed minimum with $1$ follows.
\end{proof}

\subsection{A Gaussian testing lemma}
\label{app:gaussian-testing-lemma}

\begin{lemma}[Blow-up function for equal-covariance Gaussians]
\label{lem:gaussian-blowup}
Let
\[
\begin{gathered}
P=\mathcal{N}(\mu_1,\sigma^2 I),
\qquad
Q=\mathcal{N}(\mu_0,\sigma^2 I),\\
\Delta:=\|\mu_1-\mu_0\|_2.
\end{gathered}
\]
Then for every $\kappa\in[0,1]$,
\[
\sup\{P(E): Q(E)\le \kappa\}
=
\Phi\!\left(\Phi^{-1}(\kappa)+\frac{\Delta}{\sigma}\right),
\]
where the values at $\kappa=0$ and $\kappa=1$ are understood as the continuous limits, namely $0$ and $1$, respectively.
\end{lemma}

\begin{proof}
By translation and orthogonal invariance of Gaussian measures, it is enough to consider
\[
Q=\mathcal{N}(0,\sigma^2 I),
\qquad
P=\mathcal{N}(\Delta e_1,\sigma^2 I),
\]
where $e_1$ is the first standard basis vector. The log-likelihood ratio is
\[
\log\frac{dP}{dQ}(x)
=
\frac{\Delta}{\sigma^2}x_1-\frac{\Delta^2}{2\sigma^2},
\]
which is increasing in $x_1$. By the Neyman--Pearson lemma~\citep{neyman1933problem}, among all measurable events with a fixed $Q$-mass, an event that maximizes $P(E)$ is a likelihood-ratio upper level set. Hence, in this Gaussian shift case, an optimal event at fixed $Q$-mass is a half-space
\[
E_t:=\{x:x_1\ge t\}.
\]
Choose $t$ such that $Q(E_t)=\kappa$. Under $Q$, $x_1\sim\mathcal N(0,\sigma^2)$, so
\[
\kappa=1-\Phi(t/\sigma)
\quad\text{and}\quad
 t=\sigma\Phi^{-1}(1-\kappa).
\]
Under $P$, $x_1\sim\mathcal N(\Delta,\sigma^2)$, hence
\[
P(E_t)
=
1-\Phi\!\left(\frac{t-\Delta}{\sigma}\right)
=
\Phi\!\left(\Phi^{-1}(\kappa)+\frac{\Delta}{\sigma}\right).
\]
The Neyman--Pearson lemma gives optimality, proving the identity.
\end{proof}

\subsection{Proof of Theorem~\ref{thm:direct-gaussian-rero}}
\label{app:proof-direct-gaussian-rero}

\begin{proof}
Fix $D^-\in\mathcal{Z}^{n-1}$ and an anchor $u\in\mathcal{Z}$. For each $z$ in the support of $\pi_{u,r}$, define
\[
D_z:=D^-\cup\{z\},
\qquad
D_u:=D^-\cup\{u\}.
\]
Let
\[
P_z:=\Law(M(D_z)),
\qquad
Q_u:=\Law(M(D_u)).
\]
Because the mechanism is Gaussian output perturbation,
\[
P_z=\mathcal{N}(f(D_z),\sigma^2 I),
\qquad
Q_u=\mathcal{N}(f(D_u),\sigma^2 I).
\]
Set
\[
\Delta_z:=\|f(D_z)-f(D_u)\|_2,
\qquad
\Delta_\star:=\bar\Delta(D^-).
\]
By definition, $\Delta_z\le\Delta_\star$ for all $z$ in the support. Since $d(z,u)\le r$, also $D_z\sim_r D_u$, and hence $\Delta_\star\le\Delta_2(r)$.

For each $z$, define
\[
E_z:=\{\theta:\rho(z,R(\theta,D^-))\le\eta\}.
\]
The attack success probability is
\[
p_{\mathrm{succ}}
=
\mathbb{E}_{Z\sim\pi_{u,r}}[P_Z(E_Z)].
\]
By Lemma~\ref{lem:gaussian-blowup},
\[
\begin{aligned}
P_z(E_z)
&\le
\Phi\!\left(\Phi^{-1}(Q_u(E_z))+\frac{\Delta_z}{\sigma}\right)\\
&\le
\Phi\!\left(\Phi^{-1}(Q_u(E_z))+\frac{\Delta_\star}{\sigma}\right).
\end{aligned}
\]
Therefore
\[
p_{\mathrm{succ}}
\le
\mathbb{E}_{Z\sim\pi_{u,r}}
\left[
\Phi\!\left(\Phi^{-1}(Q_u(E_Z))+\frac{\Delta_\star}{\sigma}\right)
\right].
\]
Let
\[
a:=\frac{\Delta_\star}{\sigma},
\qquad
g(t):=\Phi\!\left(\Phi^{-1}(t)+a\right).
\]
For $t\in(0,1)$, with $x=\Phi^{-1}(t)$ and $\varphi$ denoting the standard normal density,
\[
g'(t)
=
\frac{\varphi(x+a)}{\varphi(x)}
=
\exp\!\left(-ax-\frac{a^2}{2}\right)>0,
\qquad
g''(t)
=
-\frac{a}{\varphi(x)}
\exp\!\left(-ax-\frac{a^2}{2}\right)
\le 0.
\]
Thus $g$ is nondecreasing and concave on $(0,1)$, and the same conclusion holds on $[0,1]$ by continuous extension. Jensen's inequality gives
\[
p_{\mathrm{succ}}
\le
\Phi\!\left(
\Phi^{-1}\!\left(\mathbb{E}_{Z\sim\pi_{u,r}}[Q_u(E_Z)]\right)
+\frac{\Delta_\star}{\sigma}
\right).
\]
It remains to bound the expectation inside $\Phi^{-1}$. By Tonelli's theorem~\citep[Chapter~2]{folland1999real},
\[
\mathbb{E}_{Z\sim\pi_{u,r}}[Q_u(E_Z)]
=
\mathbb{E}_{\theta\sim Q_u}
\Pr_{Z\sim\pi_{u,r}}[\rho(Z,R(\theta,D^-))\le \eta].
\]
For each fixed $\theta$, the output constraint gives $R(\theta,D^-)\in\mathcal A_{u,r}$, so
\[
\Pr_{Z\sim\pi_{u,r}}[\rho(Z,R(\theta,D^-))\le \eta]
\le
\kappa_{\pi_{u,r},\rho}(\eta;\mathcal A_{u,r})
=
\kappa.
\]
Thus
\[
\mathbb{E}_{Z\sim\pi_{u,r}}[Q_u(E_Z)]\le\kappa.
\]
Using the monotonicity of $g$ yields
\[
p_{\mathrm{succ}}
\le
\Phi\!\left(\Phi^{-1}(\kappa)+\frac{\Delta_\star}{\sigma}\right).
\]
Finally, $\Delta_\star\le\Delta_2(r)$ gives the second displayed inequality in the theorem.
\end{proof}

\subsection{Proof of Corollary~\ref{cor:uniform-finite-prior}}
\label{app:proof-uniform-finite-prior}

\begin{proof}
Under the exact-identification loss $\rho_{\mathrm{ID}}(z,z')=\ind{z\neq z'}$ with tolerance $\eta=0$, the success event is exactly $R(\theta,D^-)=Z$. The allowed output set is $\mathcal A_{u,r}=S$. For the uniform prior
\[
\pi_S=\frac{1}{m}\sum_{i=1}^m\delta_{z_i},
\]
the best oblivious feasible-set strategy is to guess any one candidate, so
\[
\kappa_{\pi_S,\rho_{\mathrm{ID}}}(0;S)
=
\sup_{a\in S}\Pr_{Z\sim\pi_S}[Z=a]
=
\frac1m.
\]
Substituting this value into Theorem~\ref{thm:direct-gaussian-rero} and then using $\bar\Delta(D^-)\le\Delta_2(r)$ gives
\[
\Pr[R(\theta,D^-)=Z]
\le
\Phi\!\left(\Phi^{-1}(1/m)+\frac{\Delta_2(r)}{\sigma}\right),
\]
which is the claimed exact-identification bound.
\end{proof}

\subsection{Support-specific finite-Gaussian exact-identification bound}
\label{app:support-specific-bound}

The following bound is a sharper finite-prior statement for equal-covariance Gaussian releases. It is useful as a technical comparison point, but it is not the quantity plotted in the experiments; the main text uses the simpler direct certificate from Corollary~\ref{cor:uniform-finite-prior}.

\begin{theorem}[Support-specific finite-Gaussian exact-ID bound]
\label{thm:finite-gaussian-support-bound}
Let $P_i=\mathcal N(\mu_i,\sigma^2I)$ be the release law under candidate $z_i$, with prior weights $\pi_i>0$ and $\sum_i\pi_i=1$. Fix any reference Gaussian $Q=\mathcal N(\mu_0,\sigma^2I)$ and define
\[
d_i:=\frac{\|\mu_i-\mu_0\|_2}{\sigma}.
\]
Then every exact-identification rule satisfies
\[
\Pr[\widehat Z=Z]
\le
\max_{q_1,\ldots,q_m\ge0:\,\sum_iq_i\le1}
\sum_{i=1}^m
\pi_i\,\Phi\!\left(\Phi^{-1}(q_i)+d_i\right),
\]
with boundary cases interpreted by continuity. In particular, if $\pi_i=1/m$ and $d_i\le R$ for all $i$, then
\[
\Pr[\widehat Z=Z]
\le
\Phi\!\left(\Phi^{-1}(1/m)+R\right).
\]
\end{theorem}

\begin{proof}
Let $A_i$ be the event that the attack outputs candidate $z_i$. The events $A_1,\ldots,A_m$ are disjoint, and therefore
\[
\sum_{i=1}^m Q(A_i)\le 1.
\]
Write $q_i:=Q(A_i)$. By Lemma~\ref{lem:gaussian-blowup}, applied to $P_i=\mathcal N(\mu_i,\sigma^2 I)$ and $Q=\mathcal N(\mu_0,\sigma^2 I)$,
\[
P_i(A_i)\le \Phi\!\left(\Phi^{-1}(q_i)+d_i\right),
\qquad
 d_i=\frac{\|\mu_i-\mu_0\|_2}{\sigma}.
\]
The exact-identification success probability is
\[
p_{\mathrm{succ}}=\sum_{i=1}^m \pi_i P_i(A_i).
\]
Combining the previous two displays gives
\[
p_{\mathrm{succ}}
\le
\sum_{i=1}^m \pi_i\Phi\!\left(\Phi^{-1}(q_i)+d_i\right)
\]
for some nonnegative $q_i$ with $\sum_i q_i\le 1$. Maximizing over all such $q_i$ proves the first claim.

For the uniform-prior specialization, assume $d_i\le R$ for all $i$ and
define
\[
h_R(q):=\Phi\!\left(\Phi^{-1}(q)+R\right).
\]
For $q\in(0,1)$, with $x=\Phi^{-1}(q)$ and $\varphi$ denoting the standard normal density,
\[
h_R'(q)
=
\frac{\varphi(x+R)}{\varphi(x)}
=
\exp\!\left(-Rx-\frac{R^2}{2}\right)>0,
\qquad
h_R''(q)
=
-\frac{R}{\varphi(x)}
\exp\!\left(-Rx-\frac{R^2}{2}\right)
\le 0.
\]
Thus $h_R$ is nondecreasing and concave on $(0,1)$, and the same conclusion holds on $[0,1]$ by continuous extension. Jensen's
inequality gives
\[
\begin{aligned}
\frac{1}{m}\sum_{i=1}^m
\Phi\!\left(\Phi^{-1}(q_i)+d_i\right)
&\le
\frac{1}{m}\sum_{i=1}^m h_R(q_i) \\
&\le
h_R\!\left(\frac{1}{m}\sum_i q_i\right) \\
&=
\Phi\!\left(
\Phi^{-1}\!\left(\frac{1}{m}\sum_i q_i\right)+R
\right).
\end{aligned}
\]
Because $\sum_i q_i\le 1$ and the same map is nondecreasing,
\[
p_{\mathrm{succ}}
\le
\Phi\!\left(\Phi^{-1}(1/m)+R\right).
\]
This proves the stated finite-Gaussian support bound.
\end{proof}

\section{Proof of Proposition~\ref{prop:finite-map}}
\label{app:proof-finite-map}

\begin{proof}
Fix a finite candidate set $S=\{z_1,\dots,z_m\}$ and prior weights $\pi_1,\dots,\pi_m$. For each candidate $z_i$, the Gaussian release density is
\[
p(\widetilde\theta\mid z_i,D^-)
=
(2\pi\sigma^2)^{-p/2}
\exp\!\left(
-\frac{1}{2\sigma^2}\bigl\|\widetilde\theta-\widehat\theta(D^-\cup\{z_i\})\bigr\|_2^2
\right),
\]
where $p$ is the parameter dimension. By Bayes' rule, for each $i=1,\dots,m$,
\[
p(z_i\mid \widetilde\theta,D^-)
=
\frac{\pi_i\,p(\widetilde\theta\mid z_i,D^-)}
{\sum_{j=1}^m \pi_j\,p(\widetilde\theta\mid z_j,D^-)}
=
\frac{
\pi_i
\exp\!\left(
-\frac{1}{2\sigma^2}\bigl\|\widetilde\theta-\widehat\theta(D^-\cup\{z_i\})\bigr\|_2^2
\right)
}{
\sum_{j=1}^m
\pi_j
\exp\!\left(
-\frac{1}{2\sigma^2}\bigl\|\widetilde\theta-\widehat\theta(D^-\cup\{z_j\})\bigr\|_2^2
\right)
}.
\]
Take logarithms:
\[
\log p(z_i\mid \widetilde\theta,D^-)
=
\log \pi_i
-
\frac{1}{2\sigma^2}\bigl\|\widetilde\theta-\widehat\theta(D^-\cup\{z_i\})\bigr\|_2^2
+
C(\widetilde\theta),
\]
where $C(\widetilde\theta)$ is a constant independent of the candidate index $i$. Since the MAP rule chooses the candidate with the largest posterior probability, and since $C(\widetilde\theta)$ is the same for every candidate, any MAP estimator is of the form
\[
\arg\max_{z_i\in S}
\left\{
\log \pi_i
-
\frac{1}{2\sigma^2}\bigl\|\widetilde\theta-\widehat\theta(D^-\cup\{z_i\})\bigr\|_2^2
\right\}.
\]
If the prior is uniform, then $\log \pi_i$ is the same for every candidate and can be dropped, leaving the least-squares objective.
\end{proof}

\section{Derivations for Secondary Convex Head Constants}
\label{app:proof-secondary-lz}

\begin{proof}[Proof of Proposition~\ref{prop:secondary-lz-summary}]
We prove the three constants separately. Throughout this proof, the record metric is the label-preserving metric used in the statement. Thus finite-distance substitutions have the same label, and for a same-label pair $(z,y),(z',y)$ the distance is $\|z-z'\|_2$. Different-label pairs have infinite distance, so they impose no additional finite Lipschitz constraint in Assumption~\ref{ass:lz}. 

The derivations in the three subsections below establish all bounded-parameter and ERM-minimizer constants in Proposition~\ref{prop:secondary-lz-summary}. Substituting any of them into Theorem~\ref{thm:erm-sensitivity} gives the stated radius-$r$ sensitivity bound.
\end{proof}

\subsection{Squared Hinge}

Fix a label $y\in\{-1,1\}$ and a parameter vector $w\in\R^d$. For a feature vector $z$, define
\[
q(z):=1-yw^\top z,
\qquad
h(z):=q(z)_+=\max\{0,q(z)\}.
\]
The squared-hinge loss is
\[
\ell(w;z,y)=h(z)^2.
\]
The scalar map $u\mapsto (u_+)^2$ is differentiable everywhere: its derivative is $0$ for $u<0$, is $2u$ for $u>0$, and is also $0$ at $u=0$ because the one-sided derivatives both equal $0$. Hence, by the chain rule,
\[
\nabla_w\ell(w;z,y)
=2h(z)\nabla_w(1-yw^\top z)
=2h(z)(-yz)
=-2yzh(z).
\]
Write
\[
g(z):=\nabla_w\ell(w;z,y)=-2yzh(z).
\]
We first bound the scalar hinge difference. The map $m(a):=a_+=\max\{0,a\}$ is $1$-Lipschitz on $\R$. Indeed, if $a,b\le0$, then $m(a)=m(b)=0$; if $a,b\ge0$, then $|m(a)-m(b)|=|a-b|$; and if $a\ge0\ge b$, then $|m(a)-m(b)|=a\le a-b=|a-b|$, with the remaining mixed case symmetric. Therefore,
\begin{align*}
|h(z)-h(z')|
&\le |q(z)-q(z')| \\
&= |(1-yw^\top z)-(1-yw^\top z')| \\
&= |yw^\top(z'-z)| \\
&= |w^\top(z'-z)| \\
&\le \|w\|_2\,\|z'-z\|_2
=\|w\|_2\,\|z-z'\|_2 .
\end{align*}
Next, for any admissible feature vector $u$ with $\|u\|_2\le B$,
\begin{align*}
h(u)
&=\max\{0,1-yw^\top u\} \\
&\le 1+|w^\top u| \\
&\le 1+\|w\|_2\|u\|_2 \\
&\le 1+B\|w\|_2 .
\end{align*}
In particular, $h(z')\le1+B\|w\|_2$.

Now subtract the gradients and add and subtract $zh(z')$:
\begin{align*}
g(z)-g(z')
&=-2yzh(z)+2yz'h(z') \\
&=-2y\bigl(zh(z)-z'h(z')\bigr) \\
&=-2y\bigl(z(h(z)-h(z'))+(z-z')h(z')\bigr).
\end{align*}
Taking Euclidean norms and using $|y|=1$ gives
\begin{align*}
\|g(z)-g(z')\|_2
&\le 2\Bigl(\|z\|_2|h(z)-h(z')|+\|z-z'\|_2|h(z')|\Bigr) \\
&\le 2\Bigl(B\|w\|_2\|z-z'\|_2+(1+B\|w\|_2)\|z-z'\|_2\Bigr) \\
&=\bigl(2+4B\|w\|_2\bigr)\|z-z'\|_2 .
\end{align*}
Thus, on $\Theta_R=\{w:\|w\|_2\le R\}$,
\[
\|\nabla_w\ell(w;z,y)-\nabla_w\ell(w;z',y)\|_2
\le (2+4BR)\|z-z'\|_2,
\]
so $L_z(R)=2+4BR$ is valid.

It remains to bound the radius of ERM minimizers. Let
\[
F_D(w)=\frac1n\sum_{i=1}^n \max\{0,1-y_iw^\top z_i\}^2+\frac\lambda2\|w\|_2^2,
\]
and let $w^\star\in\arg\min_w F_D(w)$. Since $w^\star$ minimizes $F_D$,
\[
F_D(w^\star)\le F_D(0).
\]
At $w=0$, every hinge value is $\max\{0,1\}=1$, so
\[
F_D(0)=\frac1n\sum_{i=1}^n1^2+\frac\lambda2\|0\|_2^2=1.
\]
The empirical squared-hinge term is nonnegative, hence
\[
F_D(w^\star)
=\frac1n\sum_{i=1}^n \max\{0,1-y_i(w^\star)^\top z_i\}^2
+\frac\lambda2\|w^\star\|_2^2
\ge \frac\lambda2\|w^\star\|_2^2.
\]
Combining the last three displays,
\[
\frac\lambda2\|w^\star\|_2^2\le F_D(w^\star)\le F_D(0)=1,
\]
so
\[
\|w^\star\|_2\le \sqrt{\frac2\lambda}.
\]
Substituting $R=\sqrt{2/\lambda}$ into $L_z(R)=2+4BR$ gives the ERM-minimizer constant
\[
L_z\le 2+4B\sqrt{\frac2\lambda}.
\]

\subsection{Binary Logistic}
Let $y\in\{-1,1\}$, let
\[
\widetilde z:=\begin{pmatrix}z\\1\end{pmatrix}\in\R^{d+1},
\qquad
\widetilde\theta:=\begin{pmatrix}w\\b\end{pmatrix}\in\R^{d+1},
\qquad
\widetilde B:=\sqrt{B^2+1}.
\]
Then $\|\widetilde z\|_2\le \widetilde B$, and for same-label pairs,
\[
\|\widetilde z-\widetilde z'\|_2
=
\left\|\begin{pmatrix}z-z'\\0\end{pmatrix}\right\|_2
=\|z-z'\|_2.
\]
Define the score $s(\widetilde z):=\widetilde\theta^\top\widetilde z$ and the sigmoid $\sigma(u):=(1+e^{-u})^{-1}$. The binary logistic loss is
\[
\ell(\widetilde\theta;\widetilde z,y)=\log\bigl(1+\exp(-ys(\widetilde z))\bigr).
\]
With $u=-ys(\widetilde z)$, we have
\[
\frac{d}{du}\log(1+e^u)=\frac{e^u}{1+e^u}=\sigma(u),
\qquad
\nabla_{\widetilde\theta}u=-y\widetilde z.
\]
Therefore,
\[
\nabla_{\widetilde\theta}\ell(\widetilde\theta;\widetilde z,y)
=-y\sigma(-y\widetilde\theta^\top\widetilde z)\widetilde z.
\]
For two same-label examples, set
\[
\alpha:=\sigma(-y\widetilde\theta^\top\widetilde z),
\qquad
\beta:=\sigma(-y\widetilde\theta^\top\widetilde z').
\]
Since $0\le\sigma\le1$, $0\le\alpha,\beta\le1$. The gradient difference satisfies
\begin{align*}
&\left\|\nabla_{\widetilde\theta}\ell(\widetilde\theta;\widetilde z,y)
-\nabla_{\widetilde\theta}\ell(\widetilde\theta;\widetilde z',y)\right\|_2 \\
&\qquad =\left\|-y\alpha\widetilde z+y\beta\widetilde z'\right\|_2 \\
&\qquad =\left\|-\alpha\widetilde z+\beta\widetilde z'\right\|_2 .
\end{align*}
Add and subtract $\alpha\widetilde z'$:
\[
-\alpha\widetilde z+\beta\widetilde z'
=\alpha(\widetilde z'-\widetilde z)+(\beta-\alpha)\widetilde z'.
\]
Thus,
\begin{align*}
\left\|-\alpha\widetilde z+\beta\widetilde z'\right\|_2
&\le |\alpha|\,\|\widetilde z'-\widetilde z\|_2
+|\beta-\alpha|\,\|\widetilde z'\|_2 \\
&\le \|z-z'\|_2+\widetilde B|\beta-\alpha|.
\end{align*}
The sigmoid derivative is
\[
\sigma'(u)=\sigma(u)(1-\sigma(u)).
\]
For $t\in[0,1]$,
\[
t(1-t)=\frac14-\left(t-\frac12\right)^2\le\frac14,
\]
so $0\le\sigma'(u)\le1/4$ and $\sigma$ is $1/4$-Lipschitz. Therefore,
\begin{align*}
|\beta-\alpha|
&\le \frac14\left| -y\widetilde\theta^\top\widetilde z' + y\widetilde\theta^\top\widetilde z\right| \\
&=\frac14\left|\widetilde\theta^\top(\widetilde z-\widetilde z')\right| \\
&\le \frac14\|\widetilde\theta\|_2\,\|\widetilde z-\widetilde z'\|_2 \\
&=\frac14\|\widetilde\theta\|_2\,\|z-z'\|_2 .
\end{align*}
Substituting this into the gradient-difference bound gives
\[
\left\|\nabla_{\widetilde\theta}\ell(\widetilde\theta;\widetilde z,y)
-\nabla_{\widetilde\theta}\ell(\widetilde\theta;\widetilde z',y)\right\|_2
\le
\left(1+\frac{\widetilde B\|\widetilde\theta\|_2}{4}\right)\|z-z'\|_2.
\]
Hence, on $\Theta_R=\{\widetilde\theta:\|\widetilde\theta\|_2\le R\}$,
\[
L_z(R)=1+\frac{\widetilde B R}{4}
\]
is valid.

Now consider the regularized objective, with both weights and intercept regularized,
\[
F_D(\widetilde\theta)
=\frac1n\sum_{i=1}^n \ell(\widetilde\theta;\widetilde z_i,y_i)
+\frac\lambda2\|\widetilde\theta\|_2^2.
\]
Let $\widetilde\theta^\star$ be an ERM minimizer. First-order optimality gives
\[
0=\nabla F_D(\widetilde\theta^\star)
=\frac1n\sum_{i=1}^n
\nabla_{\widetilde\theta}\ell(\widetilde\theta^\star;\widetilde z_i,y_i)
+\lambda\widetilde\theta^\star,
\]
so
\[
\lambda\widetilde\theta^\star
=-\frac1n\sum_{i=1}^n
\nabla_{\widetilde\theta}\ell(\widetilde\theta^\star;\widetilde z_i,y_i).
\]
Taking norms and using the triangle inequality,
\begin{align*}
\lambda\|\widetilde\theta^\star\|_2
&\le \frac1n\sum_{i=1}^n
\left\|\nabla_{\widetilde\theta}\ell(\widetilde\theta^\star;\widetilde z_i,y_i)\right\|_2.
\end{align*}
For each example,
\begin{align*}
\left\|\nabla_{\widetilde\theta}\ell(\widetilde\theta^\star;\widetilde z_i,y_i)\right\|_2
&=\left\|-y_i\sigma(-y_i(\widetilde\theta^\star)^\top\widetilde z_i)\widetilde z_i\right\|_2 \\
&\le |y_i|\cdot 1\cdot\|\widetilde z_i\|_2 \\
&\le \widetilde B.
\end{align*}
Therefore,
\[
\lambda\|\widetilde\theta^\star\|_2
\le \frac1n\sum_{i=1}^n\widetilde B
=\widetilde B,
\]
and hence
\[
\|\widetilde\theta^\star\|_2\le\frac{\widetilde B}{\lambda}.
\]
Substituting $R=\widetilde B/\lambda$ into the fixed-parameter constant gives
\[
L_z\le 1+\frac{\widetilde B}{4}\cdot\frac{\widetilde B}{\lambda}
=1+\frac{\widetilde B^2}{4\lambda}
=1+\frac{B^2+1}{4\lambda}.
\]

\subsection{Softmax Logistic}

Let $y\in\{1,\ldots,K\}$, let $\widetilde z=(z,1)\in\R^{d+1}$, and let $\widetilde W\in\R^{K\times(d+1)}$. Define
\[
p(\widetilde z):=\softmax(\widetilde W\widetilde z),
\qquad
a(\widetilde z):=p(\widetilde z)-e_y,
\]
where $e_y$ is the $y$th standard basis vector in $\R^K$. For scores $s:=\widetilde W\widetilde z$, the softmax loss is
\[
\ell(\widetilde W;\widetilde z,y)
=-s_y+\log\left(\sum_{c=1}^K e^{s_c}\right).
\]
For each class $c$,
\[
\frac{\partial \ell}{\partial s_c}
=-\mathbf 1\{c=y\}+\frac{e^{s_c}}{\sum_{j=1}^K e^{s_j}}
=p_c(\widetilde z)-(e_y)_c.
\]
Thus $\nabla_s\ell=p(\widetilde z)-e_y=a(\widetilde z)$. Since $s_c=\widetilde w_c^\top\widetilde z$ for the $c$th row $\widetilde w_c^\top$ of $\widetilde W$, the row-wise chain rule gives
\[
\nabla_{\widetilde W}\ell(\widetilde W;\widetilde z,y)
=a(\widetilde z)\widetilde z^\top.
\]
For a same-label pair, write
\[
a:=a(\widetilde z),
\qquad
a':=a(\widetilde z').
\]
Subtract the gradients and add and subtract $a(\widetilde z')^\top$:
\begin{align*}
&\nabla_{\widetilde W}\ell(\widetilde W;\widetilde z,y)
-\nabla_{\widetilde W}\ell(\widetilde W;\widetilde z',y) \\
&\qquad = a\widetilde z^\top-a'(\widetilde z')^\top \\
&\qquad = a(\widetilde z-\widetilde z')^\top+(a-a')(\widetilde z')^\top.
\end{align*}
Taking Frobenius norms and using the triangle inequality,
\begin{align*}
&\left\|\nabla_{\widetilde W}\ell(\widetilde W;\widetilde z,y)
-\nabla_{\widetilde W}\ell(\widetilde W;\widetilde z',y)\right\|_F \\
&\qquad\le
\left\|a(\widetilde z-\widetilde z')^\top\right\|_F
+
\left\|(a-a')(\widetilde z')^\top\right\|_F.
\end{align*}
For vectors $u,v$, $\|uv^\top\|_F=\|u\|_2\|v\|_2$, so
\begin{align}
\label{eq:softmax-gradient-split}
&\left\|\nabla_{\widetilde W}\ell(\widetilde W;\widetilde z,y)
-\nabla_{\widetilde W}\ell(\widetilde W;\widetilde z',y)\right\|_F \\
&\qquad\le
\|a\|_2\,\|\widetilde z-\widetilde z'\|_2
+
\|a-a'\|_2\,\|\widetilde z'\|_2 . \notag
\end{align}
We bound the two factors separately.

First, since $p(\widetilde z)$ is a probability vector,
\begin{align*}
\|a\|_2^2
&=\|p(\widetilde z)-e_y\|_2^2 \\
&=(1-p_y(\widetilde z))^2+
\sum_{c\ne y}p_c(\widetilde z)^2 \\
&\le (1-p_y(\widetilde z))^2+
\left(\sum_{c\ne y}p_c(\widetilde z)\right)^2 \\
&=(1-p_y(\widetilde z))^2+(1-p_y(\widetilde z))^2 \\
&\le 2.
\end{align*}
Thus $\|a\|_2\le\sqrt2$.

Second, because $a-a'=p(\widetilde z)-p(\widetilde z')$, we need a Lipschitz bound for softmax. For $p=\softmax(x)$, the Jacobian is
\[
J(x)=\operatorname{Diag}(p)-pp^\top.
\]
It is symmetric. In row $i$,
\begin{align*}
\sum_{j=1}^K |J_{ij}(x)|
&=p_i(1-p_i)+\sum_{j\ne i}p_ip_j \\
&=p_i(1-p_i)+p_i\sum_{j\ne i}p_j \\
&=2p_i(1-p_i) \\
&\le \frac12.
\end{align*}
Therefore $\|J(x)\|_\infty\le1/2$. Since $J(x)$ is symmetric, $\|J(x)\|_2\le\|J(x)\|_\infty\le1/2$. Applying the mean-value theorem to $x\mapsto\softmax(x)$ along the segment between $\widetilde W\widetilde z$ and $\widetilde W\widetilde z'$ gives
\begin{align*}
\|p(\widetilde z)-p(\widetilde z')\|_2
&\le \frac12\|\widetilde W\widetilde z-\widetilde W\widetilde z'\|_2 \\
&=\frac12\|\widetilde W(\widetilde z-\widetilde z')\|_2 \\
&\le \frac12\|\widetilde W\|_F\,\|\widetilde z-\widetilde z'\|_2.
\end{align*}
Hence
\[
\|a-a'\|_2\le \frac12\|\widetilde W\|_F\,\|\widetilde z-\widetilde z'\|_2.
\]
Since $\|\widetilde z'\|_2\le\widetilde B$ and $\|\widetilde z-\widetilde z'\|_2=\|z-z'\|_2$, substituting the two bounds into~\eqref{eq:softmax-gradient-split} yields
\begin{align*}
&\left\|\nabla_{\widetilde W}\ell(\widetilde W;\widetilde z,y)
-\nabla_{\widetilde W}\ell(\widetilde W;\widetilde z',y)\right\|_F \\
&\qquad\le
\left(\sqrt2+\frac{\widetilde B\|\widetilde W\|_F}{2}\right)\|z-z'\|_2.
\end{align*}
Therefore, on $\Theta_R=\{\widetilde W:\|\widetilde W\|_F\le R\}$,
\[
L_z(R)=\sqrt2+\frac{\widetilde B R}{2}
\]
is valid.

It remains to bound the Frobenius norm of regularized ERM minimizers. Let
\[
F_D(\widetilde W)
=\frac1n\sum_{i=1}^n\ell(\widetilde W;\widetilde z_i,y_i)
+\frac\lambda2\|\widetilde W\|_F^2,
\]
and let $\widetilde W^\star$ minimize this objective. First-order optimality gives
\[
0=\frac1n\sum_{i=1}^n
\nabla_{\widetilde W}\ell(\widetilde W^\star;\widetilde z_i,y_i)
+\lambda\widetilde W^\star,
\]
so
\[
\lambda\widetilde W^\star
=-\frac1n\sum_{i=1}^n
\nabla_{\widetilde W}\ell(\widetilde W^\star;\widetilde z_i,y_i).
\]
Taking Frobenius norms and applying the triangle inequality,
\[
\lambda\|\widetilde W^\star\|_F
\le \frac1n\sum_{i=1}^n
\left\|\nabla_{\widetilde W}\ell(\widetilde W^\star;\widetilde z_i,y_i)\right\|_F.
\]
For one example, define $p_\star(\widetilde z_i):=\softmax(\widetilde W^\star\widetilde z_i)$. Then
\[
\nabla_{\widetilde W}\ell(\widetilde W^\star;\widetilde z_i,y_i)
=\bigl(p_\star(\widetilde z_i)-e_{y_i}\bigr)\widetilde z_i^\top,
\]
and hence, by the same probability-vector bound used above,
\[
\left\|\nabla_{\widetilde W}\ell(\widetilde W^\star;\widetilde z_i,y_i)\right\|_F
=\|p_\star(\widetilde z_i)-e_{y_i}\|_2\,\|\widetilde z_i\|_2
\le \sqrt2\,\widetilde B.
\]
Therefore,
\[
\lambda\|\widetilde W^\star\|_F
\le \frac1n\sum_{i=1}^n\sqrt2\,\widetilde B
=\sqrt2\,\widetilde B,
\]
so every ERM minimizer satisfies
\[
\|\widetilde W^\star\|_F\le\frac{\sqrt2\,\widetilde B}{\lambda}.
\]
Substituting $R=\sqrt2\,\widetilde B/\lambda$ into the bounded-parameter constant gives
\begin{align*}
L_z
&\le \sqrt2+\frac{\widetilde B}{2}\cdot\frac{\sqrt2\,\widetilde B}{\lambda} \\
&=\sqrt2\left(1+\frac{\widetilde B^2}{2\lambda}\right) \\
&=\sqrt2\left(1+\frac{B^2+1}{2\lambda}\right).
\end{align*}

\end{document}